\pdfoutput=1
\documentclass{article}
\usepackage{microtype}
\usepackage{epsfig,psfrag}
\usepackage[subfigure]{tocloft}
\usepackage{booktabs} % for professional tables
\usepackage[table]{xcolor}
\usepackage{tablefootnote}
\usepackage{caption}
\usepackage{subfig}
\usepackage[accepted]{icml2024}

\usepackage{amsmath}
\usepackage{amssymb}
\usepackage{mathtools}
\usepackage{dsfont}
\usepackage{pdfpages}
\usepackage{enumitem}
\usepackage{url}
\usepackage{multicol}
\usepackage{amsthm}
\usepackage{graphicx}
\usepackage{multirow}
\usepackage{makecell}
\usepackage{float}
\usepackage{wrapfig}  
\usepackage{lipsum}  
\usepackage{dsfont}
\usepackage{hyperref}
\usepackage[capitalize,noabbrev]{cleveref}
\theoremstyle{plain}
\newtheorem{theorem}{Theorem}[section]

\newtheorem{lemma}[theorem]{Lemma}

\theoremstyle{definition}

\newtheorem{property}[theorem]{Property}
\newtheorem{example}[theorem]{Example}
\theoremstyle{remark}

\usepackage[textsize=tiny]{todonotes}

\icmltitlerunning{Improving Group Robustness on Spurious Correlation Requires Preciser Group Inference}

\begin{document}

\twocolumn[
\icmltitle{Improving Group Robustness on Spurious Correlation\\Requires Preciser Group Inference}

% \icmltitle{Improving Group Robustness on Spurious Correlation\\ via Preciser Group Inference}


\icmlsetsymbol{equal}{*}

\begin{icmlauthorlist}
\icmlauthor{Yujin Han}{sch1}
% \icmlauthor{Firstname1 Lastname1}{equal,yyy}
% \icmlauthor{Firstname2 Lastname2}{equal,yyy,comp}
% \icmlauthor{Firstname3 Lastname3}{comp}
\icmlauthor{Difan Zou}{sch2}
% \icmlauthor{Firstname8 Lastname8}{yyy}
\end{icmlauthorlist}

\icmlaffiliation{sch1}{Department of Computer Science, The University of Hong Kong}
\icmlaffiliation{sch2}{Department of Computer Science \& Institute of Data Science, The University of Hong Kong}

\icmlcorrespondingauthor{Difan Zou}{dzou@cs.hku.hk}

\icmlkeywords{Machine Learning, ICML}
\vskip 0.3in
]



\printAffiliationsAndNotice{}  % leave blank if no need to mention equal contribution

\begin{abstract}
Standard empirical risk minimization (ERM) models may prioritize learning spurious correlations between spurious features and true labels, leading to poor accuracy on groups where these correlations do not hold. Mitigating this issue often requires expensive spurious attribute (group) labels or relies on trained ERM models to infer group labels when group information is unavailable. 
However, the significant performance gap in worst-group accuracy between using pseudo group labels and using oracle group labels inspires us to consider further improving group robustness through preciser group inference. Therefore, we propose GIC, a novel method that accurately infers group labels, resulting in improved worst-group performance. GIC trains a spurious attribute classifier based on two key properties of spurious correlations: (1) high correlation between spurious attributes and true labels, and (2) variability in this correlation between datasets with different group distributions. Empirical studies on multiple datasets demonstrate the effectiveness of GIC in inferring group labels, and combining GIC with various downstream invariant learning methods improves worst-group accuracy, showcasing its powerful flexibility. Additionally, through analyzing the misclassifications in GIC, we identify an interesting phenomenon called semantic consistency, which may contribute to better decoupling the association between spurious attributes and labels, thereby mitigating spurious correlation. The code for GIC is available at \href{https://github.com/yujinhanml/GIC}{https://github.com/yujinhanml/GIC}.
\end{abstract}

\section{Introduction}
\label{Introduction}
The presence of spurious correlation causes machine learning models to fail on certain groups of samples, even when achieving high accuracy on average group \cite{ribeiro2016should,beery2018recognition,hashimoto2018fairness,duchi2019distributionally}. Deep networks trained via standard empirical risk minimization (ERM) may be biased by the spurious attributes/features such as the image backgrounds \citep{beery2018recognition,geirhos2020shortcut}, which admit different correlations with the true labels for the different groups of data. As a result, the model typically performs well on the majority group of data that aligns with the spurious attributes but performs poorly on certain groups of data that have no or even opposite correlations with spurious attributes. This critical issue is widespread in many fields, such as medical AI, facial recognition and sentiment analysis \cite{blodgett2016demographic,tatman2017gender,gururangan2018annotation,badgeley2019deep,sagawa2019distributionally}.
% are 
% For example, in the task of classifying landbirds or waterbirds, where the background of most waterbirds is water and that of landbirds is land, models trained via standard empirical risk minimization (ERM) to identify bird types by using spurious background attributes (water or land) instead of focusing on birds themselves. This spurious correlation between the spurious attribute (background) and true labels causes the model to fail on certain groups, such as waterbirds with land backgrounds. 

There have emerged a series of previous works that aim to develop more robust training methods to improve the worst-group performance when group information is available.  For instance, one can focus on particularly training over the worst-group risk \citep{sagawa2019distributionally} via group distributional robust optimization (GroupDRO). Additionally, the spurious correlation can also be eliminated by balancing the risk in different groups, which can be achieved via feature reweighting \citep{kirichenko2022last} or selective Mixup \citep{yao2022improving}. These methods have been demonstrated to be effective to mitigate the spurious correlations, leading to substantial improvements over ERM. 
% To address the issue of low worst-group accuracy (e.g. misclassifying waterbirds with land background) due to spurious correlations, previous approaches typically assume that group labels are available and propose different methods to learn the invariant features of samples (e.g.,the birds themselves) and train a more robust model. For example, some methods aim to minimize the empirical worst-group risk through group distributionally robust optimization \cite{sagawa2019distributionally} or construct training data with balanced groups by Subsample via deep feature reweighting \cite{kirichenko2022last}.

However, group labeling is expensive, labor-intensive, and human-biased. It is more important and practical to develop robust algorithms for improving worst-group accuracy when the group information is unavailable \citep{liu2021just,zhang2022correct}. In such a scenario, a line of recent works 
\cite{sohoni2020no,nam2020learning,ahmed2020systematic,liu2021just,zhang2022correct,chen2023understanding} propose to first infer the data groups, and then seek to train a robust model according to the estimated group information. For instance, JTT \cite{liu2021just} identifies the data being misclassified by the ERM as the minority group, and trains a more robust model via upweighting the examples in minority groups; CnC \cite{zhang2022correct} infers pseudo group labels via ERM and develops a contrastive learning method across groups to train the final model.

The primary focus of the aforementioned research is on invariant learning, i.e., leveraging the group information obtained from ERM to learn the invariant attributes. However, the accuracy of group label inference has been overlooked, making the existing group inference-based methods perform significantly worse than those using oracle group labels. Very recently, some group inference methods have been developed as alternatives to ERM, such as EIIL \citep{creager2021environment}, SSA \citep{nam2022spread}, ZIN \cite{lin2022zin} and DISC \cite{wu2023discover}. However, they may either struggle with complex spurious correlations (e.g., EIIL) \citep{lin2022zin}, or require additional information (e.g., human inspection, few spurious annotations, etc) that is related to the group labels (e.g., SSA, ZIN, and DISC). Therefore, noticing the performance shortcomings and applicability limitations in existing group inference methods, we raise an important yet challenging question: 

\textit{Can we develop a more accurate group inference method to mitigate spurious correlations without relying on any additional information?}

In this work, we propose GIC (Group Inference via data Comparison), a novel, practical and accurate method to infer group labels, which can be then seamlessly incorporated into a variety of oracle group label-based robust learning methods to improve the worst-group performance. 
In particular, GIC adopts a (unlabeled) comparison dataset, which has (slightly) different group distribution compared to the training dataset, and infers the group information in a contrastive manner. More specifically, GIC trains an alternative classifier to assign spurious attribute labels to data points solely on their spurious attributes, where the training objective is designed to encourage discovering the distribution discrepancy between the comparison and training dataset while maintaining the high correlation between the spurious attribute label and true label in the training dataset. Notably, unlike previous methods that may rely on (partial) group labels or human inspections, obtaining the comparison data does not necessitate any assumptions. It does not need to be labeled and can be obtained from various sources, such as a validation set with true validation labels, a completely unlabeled subset sampled from the test data, or even directly 
resampled from the training dataset in a non-uniform manner. We highlight the main contributions as follows:

\begin{enumerate}[leftmargin=*,nosep]
\item We propose GIC, a principled method for more accurate group inference. It encourages the high correlation between predicted spurious attribute labels and true labels on the training set, while emphasizing the differences in this correlation between training and comparison data. Compared with existing group inference methods, GIC consistently achieves higher recall and precision in predicting groups for various datasets (see \cref{The scalability of GIC}). 
\item We show that the proposed GIC can seamlessly integrate with multiple invariant learning algorithms to improve the worst-group accuracy. In \cref{Main Results} and \ref{The scalability of GIC}, GIC is successfully combined with Mixup \cite{yao2022improving}, GroupDRO \cite{creager2021environment}, Upsample \cite{liu2021just}, and Subsample \cite{kirichenko2022last}. It can be seen that GIC consistently outperform baselines in terms of the worst-group accuracy, for all candidate invariant learning algorithms. More importantly, when integrated with Mixup, our model improves over the state-of-the-art for nearly all tasks when the group label is unavailable. Additionally, the average and worst-group accuracy of our model can almost match that of using oracle group labels directly. This further justifies the effectiveness of our group inference method.
\item We illustrate that GIC can infer reasonable groups that differ from human decisions in \cref{Error Case}. Analysis of misclassified examples reveals GIC's semantic consistency, where similar semantic instances are assigned to the same group, although they are not categorized into the same group by human decisions. This semantic consistency benefits methods like Mixup, which rely on distorting semantics for invariant learning, leading to improved worst-group accuracy compared to using oracle group labels. It highlights the potential effectiveness and improvements of integrating GIC with human decisions when group information defined by humans is accessible.
\end{enumerate}
\section{Related Work}
\label{Related Work}
Recent research shows that the traditional ERM can learn both spurious and invariant features \cite{kirichenko2022last,izmailov2022feature,rosenfeld2022domain,chen2023understanding}. However, with strong spurious correlations, ERM tends to prioritize learning the spurious features \cite{chen2023understanding}, which hinders its ability to generalize on data where these spurious correlations are absent. To tackle this issue, invariant feature learning, referred to as invariant learning for simplicity, has been introduced to learn invariant representations. A notable approach to invariant learning is the use of group robustness methods.

When group labels are available, some classical group robustness methods, such as GroupDRO \cite{sagawa2019distributionally}, attempt to minimize the worst-group loss instead of the average loss using oracle group labels. Other methods aim to achieve invariant learning by balancing the majority and minority groups, such as reweighting \citep{sagawa2020investigation}, regularization \citep{cao2019learning}, and downsampling \citep{kirichenko2022last}. Additionally, approaches like semi-supervised learning \cite{nam2022spread, sohoni2021barack} or Mixup \citep{yao2022improving}, which selectively combine samples with matching labels but differing spurious attributes or matching spurious attributes but differing labels, are also be considered to improve the the worst-group accuracy. Some methods attempt to infer group labels by training a simple ERM \cite{blodgett2016demographic,tatman2017gender,gururangan2018annotation,badgeley2019deep,sagawa2019distributionally}, maximizing the GroupDRO loss \cite{creager2021environment}, or even introducing human knowledge \cite{lin2022zin,wu2023discover} when group information is unknown. However, these group inferred 
 methods have performance gaps compared to group annotation utilized methods and may not be applicable when prior information is unavailable. In this work, we focus on group robustness without relying on any group labels or human-provided information.

We point out that, although we aims to learn invariant features to train more robust and generalizable models, similar to classical domain adaptation \cite{blanchard2011generalizing,muandet2013domain} methods, such as UDA \cite{ganin2015unsupervised} and DANN \cite{ganin2016domain}, our setting differs from them which divide the training data into source and target domains and use the source and target features, along with the source labels, to transfer knowledge to the specific target domain \cite{zhang2022correct}. In contrast, we do not have a natural source and target domains and strive to enhance the accuracy of specific groups affected by spurious correlations, without prior knowledge of training sample domains or spurious attributes.

\section{Method}
\label{Method}

\subsection{Problem Setup}
Consider the training dataset $\mathcal{D}$, which comprises $n$ data point-label pairs, where each data point-label pair belongs to some group $g \in \mathcal{G}$, denoted as $\mathcal{D} = \{(\mathbf{x}_i, y_i, g_i)\}_{i=1}^n$. Following previous work \cite{liu2021just, zhang2022correct, yao2022improving, yang2023change}, we consider each group \(g = (y, a)\) to be jointly defined by the label \(y \in \mathcal{Y}\) and unobserved spurious attributes \(a \in \mathcal{A}\) (i.e., \(\mathcal{G} = \mathcal{Y} \times \mathcal{A}\)), where the spurious attribute \(a\) is spuriously correlated with the label \(y\) (e.g., the image background and the image label). We denote \(y_s\) as the spurious attribute label of \(a\) (e.g., the label of the image background, which can be ``land'' versus ``water'' or ``red'' versus ``green''). The groups $g$ in dataset $\mathcal{D}$ often encounter the imbalance issue, where ERM training tends to mostly depend on majority groups, thus may memorize the spurious correlation contributed by these groups. However, such spurious correlations are often absent or even oppositely appearing in minority groups. Therefore, our ultimate objective is to train a robust model $f_{\mathrm{robust}}(\theta)$, parameterized by $\theta \in {\Theta}$, capable of classifying an input $\mathbf{x}_i$ to a label $y_i$ regardless of whether it belongs to the majority or minority group. This is referred to as the worst-group accuracy defined as follows:
\begin{equation}
\label{eq:worst acc}
    {\max_{g \in \mathcal{G}} \mathbb{E}[\mathds{1}[f_{\theta}(\mathbf{x}) \neq y]|g]},
\end{equation}
where $\mathcal G$ denotes the set of groups. In this paper, we consider the setting in that the group information is unavailable, one can only train the robust classifier $f_{\mathrm{robust}}$ using the ungrouped data. Moreover, instead of using raw data point $\mathbf x$, we can also make use of its feature representation, which is obtained by training a model from scratch via ERM or directly using a pretrained model. We define the feature representation of $\mathbf x$ as follows:
% \vspace{-0.1cm}
\begin{equation}
\label{eq:reprentation}
\mathbf{z} = \Phi(\mathbf{x}),
% \vspace{-0.1cm}
\end{equation}
where $\Phi(\cdot)$ is a feature extractor.

\subsection{GIC: Group Inference via data Comparison}
\label{sec:GIC: Group Inference via data Comparison}
We now introduce GIC (\textbf{G}roups \textbf{I}nference via data \textbf{C}omparison), a principled and novel method that infers spurious attribute (group) labels and contributes to mitigating spurious correlations. The \textit{goal of GIC} is to train a spurious attribute classifier to predict the spurious attribute label \(y_s\), i.e.,
\begin{equation}
\label{eq:goal}
\hat{y}_{s,\mathbf{w}} = f_{\mathrm{GIC}}(\mathbf{z};\mathbf{w}),
\end{equation}
where $\mathbf{w}$ is weights of GIC model $f_{\mathrm{GIC}}$. Then using the predicted $\hat{y}_{s,\mathbf{w}}$, GIC partitions the data into different groups, i.e., $\hat g = (y, \hat{y}_{s,\mathbf{w}})$. The inferred group $\hat g$ can be integrated into downstream invariant learning methods, aiding in training the robust model $f_{\mathrm{robust}}$. 

\textbf{Comparison data.} Before formally introducing details of GIC, we first introduce the concept of comparison data. We assume access to a dataset \(\mathcal{C}\), where its group distribution (slightly) differs from the training data \(\mathcal{D}\), e.g., the group distribution of \(\mathcal{D}\) is \((g_1, g_2) = (0.1, 0.9)\), while the group distribution of \(\mathcal{C}\) is \((g_1, g_2) = (0.2, 0.8)\). We refer to \(\mathcal{C}\) as \textit{comparison data}, which can have true labels or be unlabeled. We remark that the comparison data is easy to obtain, which will be thoroughly discussed in \cref{The Construction of Comparison Data}.

Based on the definition of spurious attribute labels $y_s$, an ideal spurious attribute prediction $\hat{y}_{s,\mathbf{w}}$ should exhibit the following two properties: (1) It should be highly correlated with the true label $y$. It is this high correlation between the spurious attribute label and the true label that biases the ERM-based model training. (2) The correlation between $\hat{y}_{s,\mathbf{w}}$ and $y$ varies across different datasets. Unlike the invariant attribute label, which is always equivalent to the true label, the correlation between the spurious attribute label and the true label is spurious, changing with the spurious attribute distribution across different datasets. Consider an extreme case, the spurious background attribute for both waterbirds and landbirds in the comparison data accounts for 50\%, while in the training set, the spurious attribute label is completely equivalent to the true label (e.g., \(y_s = y\)). The equivalence correlation that exists in the training data does not hold in the comparison data.

Based on the mentioned two properties, we design the optimization objective of GIC, which consists of two terms: (1) \textbf{Correlation Term}, used to describe the high correlation between $\hat{y}_{s,\mathbf{w}}$ and $y$ in the training data; (2) \textbf{Spurious Term}, used to describe the discrepancy in the correlation between the training and comparison data, aiming to emphasize the correlation is spurious rather than invariant. By jointly optimizing the above objectives, GIC is encouraged to train a spurious attribute classifier that yields the spurious attribute prediction $\hat{y}_{s,\mathbf{w}}$, as shown in Equation \eqref{eq:goal}.

\textbf{Correlation Term.} To encourage the high correlation between $y$ and $\hat{y}_{s,\mathbf{w}}$ in the training set, we consider the following optimization objective:
% \vspace{-0.1cm}
\begin{equation}
\label{eq:correlation}
    \max_{\mathbf{w}} I(y^{tr};\hat y^{tr}_{s,\mathbf{w}}),
% \vspace{-0.1cm}
\end{equation}
where $\hat y^{tr}_{s,\mathbf{w}}=f_\mathrm{GIC}(\mathbf{z}^{tr};\mathbf{w})$ is the predicted spurious attribute label on training data. Here, mutual information is considered due to its widespread usage to measure the correlation or dependency between random variables \cite{li2016mutual, kong2019mutual, pan2020adversarial, su2023towards}. We aim to maximize the mutual information between $y_s$ and $y$, encouraging the predicted spurious attribute label $y_s$ from GIC to be highly correlated with the true label $y$.

Solely satisfying Equation \eqref{eq:correlation} is not enough as invariant attribute information contained in feature representations \cite{kirichenko2022last,chen2023understanding} can also exhibit high correlation with true labels. To aid GIC in learning spurious attributes instead of invariant ones, we introduce an additional spurious term.

\textbf{Spurious Term.} We use conditional probability to describe the correlation between $\hat{y}_{s,\mathbf{w}}$ and $y$, i.e., $\mathbb{P}(y|\hat{y}_{s,\mathbf{w}})$. Then let $\mathbb P(y^{tr}|\hat y^{tr}_{s})$ and $\mathbb P(y^{c}|\hat y^{c}_{s})$ be the conditional distributions in the training and comparison dataset respectively, we will characterize their discrepancy. Since we are comparing distributions with the same support, we opt for the simplest and widely used choice  \cite{ahmed2020systematic,lee2022diversify}, the KL-divergence. Then, we maximize the following objective to encourage GIC to learn spurious attributes:
\begin{equation}
\label{eq:spurious}
    \max_{\mathbf{w}} \mathrm{KL}(\mathbb{P}(y^{tr}|\hat y^{tr}_{s,\mathbf{w}})|| \mathbb{P}(y^{c}|\hat y^{c}_{s,\mathbf{w}})),
\end{equation}
where $\hat y^{c}_{s,\mathbf{w}} = f_\mathrm{GIC}(\mathbf{z}^{c};\mathbf{w})$ is the estimated spurious attribute labels on the comparison data. 

Note that a more formal expression for Equation \eqref{eq:spurious} should use different distribution notations and the same variable names, i.e.,
\begin{equation}
\max_{\mathbf{w}} \mathrm{KL}(\mathbb{P}_{\mathrm{tr}}(y|\hat y_{s,\mathbf{w}})|| \mathbb{P}_{\mathrm{c}}(y|\hat y_{s,\mathbf{w}})) 
\end{equation}
where $\mathbb{P}_{tr}(y|\hat y_{s,\mathbf{w}}):=\mathbb P(y^{\mathrm{tr}}|\hat y^{tr}_{s,\mathbf{w}})$ and $\mathbb{P}_{\mathrm{c}}(y|\hat y_{s,\mathbf{w}}):=\mathbb P(y^{c}|\hat y^{c}_{s,\mathbf{w}})$. Here, distinct variable names, such as $\hat y^{tr}_{s,\mathbf{w}}$ and $\hat y^{c}_{s,\mathbf{w}}$, are utilized to indicate their origins from different (training or comparison) datasets, thereby emphasizing GIC can achieve improved group inference via data comparison.

The essence of maximizing Equation \eqref{eq:spurious} is to encourage GIC to learn spurious attributes by violating the invariant learning principle \cite{creager2021environment}. If $\hat{y}_{s,\mathbf{w}}$ is inferred based on invariant attributes rather than spurious ones, then $\mathbb{P}(y^{tr}|\hat y^{tr}_{s,\mathbf{w}})$ and $\mathbb{P}(y^{c}|\hat y^{c}_{s,\mathbf{w}})$ will be identical and then
\begin{equation}
\label{eq:spurious-1}
\mathrm{KL}(\mathbb{P}(y^{tr}|\hat y^{tr}_{s,\mathbf{w}})|| \mathbb{P}(y^{c}|\hat y^{c}_{s,\mathbf{w}})) = 0.
% \vspace{-0.1cm}
\end{equation}
Conversely, if $\hat{y}_{s,\mathbf{w}}$ is based on spurious attributes rather than invariant attributes, the inequality 
\begin{equation}
\label{eq:spurious-1}
\mathrm{KL}(\mathbb{P}(y^{tr}|\hat y^{tr}_{s,\mathbf{w}})|| \mathbb{P}(y^{c}|\hat y^{c}_{s,\mathbf{w}})) \ge 0
% \vspace{-0.1cm}
\end{equation}
holds, where the equality holds only when the training and comparison data share the same group distribution. Therefore, maximizing Equation \eqref{eq:spurious} entails encouraging GIC to prioritize learning spurious attributes over invariant ones.
% \vspace{-0.1cm}
% \begin{algorithm}
% \caption{GIC}
% \label{alg:GIC}
% \textbf{Input:} Training data \(\mathcal{D}\); comparison data \(\mathcal{C}\); feature extractor $\Phi(\cdot)$; weighting parameters $\gamma$; training epochs $K$ of GIC \\
% \textbf{Stage 1: Extracting feature representations}\\
% Obtain $\mathbf{z}^{tr} = \Phi(\mathbf{x^{tr}}), \mathbf{z}^{c} = \Phi(\mathbf{x}^{c})$ where $\mathbf{x^{c}} \in \mathcal{D}, \mathbf{x}^{c} \in \mathcal{C}$.\\ 
% \textbf{Stage 2: Inferring group labels}\\
% Initialize the parameters $\mathbf{w}$ for spurious attribute classifier $f_\mathrm{GIC}$.
% \begin{algorithmic}[1]
% \For{epoch $1$ to $K$}
% \If{the true label of $\mathcal{C}$ is available}
% \State Optimizing Equation \eqref{eq:GIC with y^val} to update $\mathbf{w}$. %$f_\mathrm{GIC}$.
% \Else
% \State Optimizing Equation \eqref{eq:GIC without y^val} to update $\mathbf{w}$. %$f_\mathrm{GIC}$.
% \EndIf
% \EndFor\\
% % \textbf{return} $f_\mathrm{GIC}$\\
% Infer spurious attribute labels $\hat{y}^{tr}_{s,\mathbf{w}} = f_{\mathrm{GIC}}(\mathbf{z}^{tr};\mathbf{w})$.\\
% \textbf{return} group labels $\hat g = (y^{tr}, \hat{y}^{tr}_{s,\mathbf{w}})$.
% \end{algorithmic}
% \end{algorithm}
\begin{algorithm}[]
\caption{GIC}
\label{alg:GIC}
\begin{algorithmic}
\STATE {\bfseries Input:} Training data $\mathcal{D}$; comparison data $\mathcal{C}$; feature extractor $\Phi(\cdot)$; weighting parameters $\gamma$; training epochs $K$ of GIC
\STATE {\bfseries Stage 1: Extracting feature representations}
\STATE Obtain $\mathbf{z}^{tr} = \Phi(\mathbf{x}^{tr}), \mathbf{z}^{c} = \Phi(\mathbf{x}^{c})$ where $\mathbf{x}^{tr} \in \mathcal{D}, \mathbf{x}^{c} \in \mathcal{C}$.
\STATE {\bfseries Stage 2: Inferring group labels}
\STATE Initialize the parameters $\mathbf{w}$ for spurious attribute classifier $f_{\mathrm{GIC}}$.
\FOR{epoch $1$ to $K$}
\IF{the true label of $\mathcal{C}$ is available}
\STATE Optimizing Equation \eqref{eq:GIC with y^val} to update $\mathbf{w}$.
\ELSE
\STATE Optimizing Equation \eqref{eq:GIC without y^val} to update $\mathbf{w}$.
\ENDIF
\ENDFOR
\STATE Infer spurious attribute labels $\hat{y}^{tr}_{s,\mathbf{w}} = f_{\mathrm{GIC}}(\mathbf{z}^{tr};\mathbf{w})$.
\STATE {\bfseries Return:} Pseudo group labels $\hat{g} = (y^{tr}, \hat{y}^{tr}_{s,\mathbf{w}})$.
\end{algorithmic}
\end{algorithm}

\textbf{Overall Objective.} By combining Equation \eqref{eq:correlation} and \eqref{eq:spurious}, we derive the overall objective of GIC as follows:
% \vspace{-0.1cm}
\begin{equation}
\label{eq:theoretical optimization objective}
    \max_{\mathbf{w}} I(y^{tr};\hat y^{tr}_{s,\mathbf{w}}) + \gamma \mathrm{KL}(\mathbb{P}(y^{tr}|\hat y^{tr}_{s,\mathbf{w}})|| \mathbb{P}(y^{c}|\hat y^{c}_{s,\mathbf{w}})),
    % \vspace{-0.1cm} 
\end{equation}
where $\gamma \geq 0$ is a weighting parameter used to balance Correlation Term and Spurious Term.
\begin{figure*}
  \centering
  \begin{minipage}{0.48\textwidth}
    \centering
    \includegraphics[width=1\textwidth, height=0.16\textheight]
    {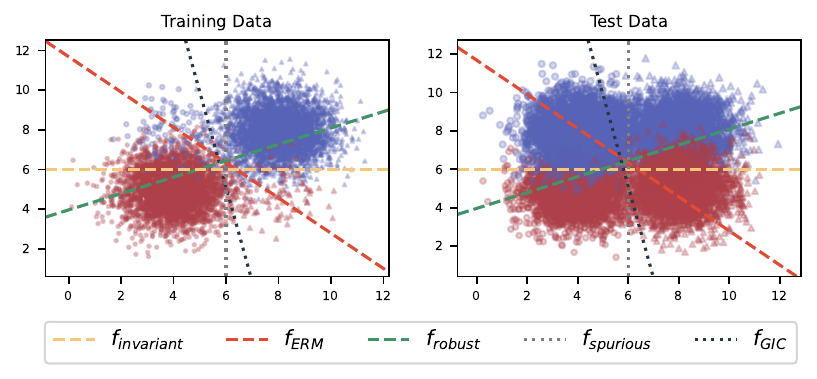}
    % \vspace*{-10mm}
    \caption{Decision boundary visualization.  $f_{\mathrm{ERM}}$ underperforms $f_{\mathrm{GIC}}$ in recognizing spurious attributes and  $f_{\mathrm{robust}}$ in identifying invariant attributes. Classes $0$ and  $1$ are represented by colors (red and blue), respectively, with shapes marking spurious attributes.
    % GIC helps obtain a robust model with decision boundaries closer to those of the invariant attribute-based model by learning spurious attributes.
    }
    \label{fig:toy example}
  \end{minipage}%
  \hspace{0.02\textwidth} 
  \begin{minipage}{0.49\textwidth}
    \centering
    \includegraphics[width=1\textwidth, height=0.16\textheight]
    {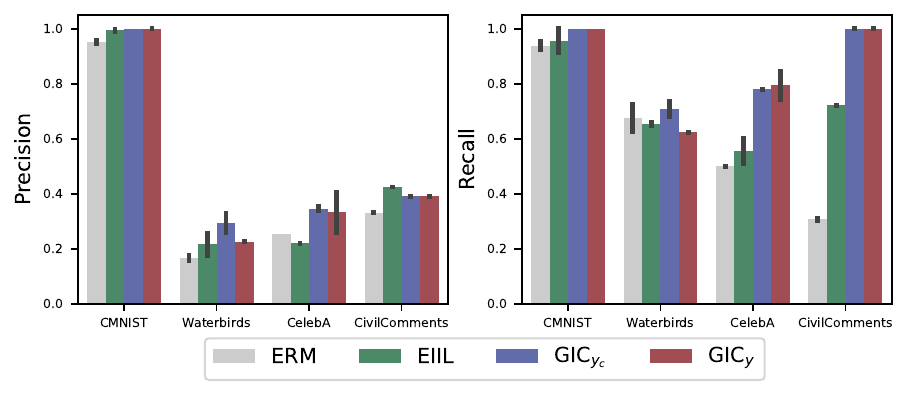}
    % \vspace{-2mm}
    \caption{Evaluation of group label inference. Compared to baseline methods such as ERM and EIIL, GIC significantly improves the recall for minority group label inference while maintaining a relatively high precision.}
    \label{fig:sacc}
  \end{minipage}
  % \vspace{-0.3cm}
\end{figure*}

However, the overall objective of GIC faces certain practical issues. Firstly, the mutual information term is difficult to accurately estimate \citep{paninski2003estimation, belghazi2018mutual}. Secondly, Equation \eqref{eq:theoretical optimization objective} cannot handle the situation where the comparison data is unlabeled, which limits the applicability of GIC in various scenarios.

We first replace the mutual information $I(y^{tr};\hat y^{tr}_{s,\mathbf{w}})$ with the cross-entropy $H(y^{tr},\hat y^{tr}_{s,\mathbf{w}})$ to achieve accurate estimation. A detailed proof in Appendix \ref{proof of Lower Bound of Correlation Term} demonstrates that $-H(y^{tr},\hat y^{tr}_{s,\mathbf{w}})$ is, in fact, a lower bound of $I(y^{tr};\hat y^{tr}_{s,\mathbf{w}})$. Therefore, maximizing $I(y^{tr};\hat y^{tr}_{s,\mathbf{w}})$ can be achieved by minimizing $H(y^{tr},\hat y^{tr}_{s,\mathbf{w}})$. This replacement aligns with intuition because maximizing mutual information between $y^{tr}$ and $\hat y^{tr}_{s,\mathbf{w}}$ essentially encourages a closer alignment of their distributions, which is consistent with the objective of minimizing cross-entropy $H(y^{tr},\hat y^{tr}_{s,\mathbf{w}})$.

We then extend GIC to the case where the true label $y^c$ of the comparison data are not available. We substitute the true labels $y^{c}$ with the comparison data's feature representation $\mathbf{z}^{c}$, which strongly associates with $y^{c}$ and is always accessible. We present the following theorem:
\begin{theorem}
\label{Lower Bound of Spurious Term without Accessible Y}
    [Lower Bound of Spurious Term without $y^{c}$]
    Given representations $\mathbf{z}^{tr}$ and $\mathbf{z}^{c}$, the spurious term is lower bounded by the following expression as:
     \begin{equation}
    \label{eq:GIC without y^val}
    \mathrm{KL}(\mathbb{P}(y^{tr}|\hat y^{tr}_{s,\mathbf{w}})|| \mathbb{P}(y^{c}|\hat y^{c}_{s,\mathbf{w}}) \geq \mathrm{KL}(\mathbb{P}(\mathbf{z}^{tr}|\hat y^{tr}_{s,\mathbf{w}})|| \mathbb{P}(\mathbf{z}^{c}|\hat y^{c}_{s,\mathbf{w}}))
    % \vspace{-0.1cm}
\end{equation}
\end{theorem}
In Theorem \ref{Lower Bound of Spurious Term without Accessible Y}, when $y^{c}$ is missing, we resort to maximizing the lower bound $\mathrm{KL}(\mathbb{P}(\mathbf{z}^{tr}|\hat y^{tr}_{s,\mathbf{w}})|| \mathbb{P}(\mathbf{z}^{c}|\hat y^{c}_{s,\mathbf{w}}))$ as an alternative. We point out that maximizing a lower bound is meaningful as it provides the worst-case guarantee over the original objective. The detailed proof of Theorem \ref{Lower Bound of Spurious Term without Accessible Y} is provided in Appendix \ref{app:proof for theorem 1}. 

Therefore, based on whether $y^c$ is accessible, we propose the following two optimization objectives:

\textbf{With $y^c$.} The overall objective of GIC can be defined as:
% \vspace{-0.1cm}
\begin{equation}
    \label{eq:GIC with y^val}
    \min_{\mathbf{w}} H(y^{tr},\hat y^{tr}_{s,\mathbf{w}})- \gamma \mathrm{KL}(\mathbb{P}(\mathbf{y}^{tr}|\hat y^{tr}_{s,\mathbf{w}})|| \mathbb{P}(\mathbf{y}^{c}|\hat y^{c}_{s,\mathbf{w}})).
    % \vspace{-0.1cm}
\end{equation}
\textbf{Without $y^c$.} The overall objective of GIC can be redefined as:
% \vspace{-0.1cm}
\begin{equation}
    \label{eq:GIC without y^val}
    \min_{\mathbf{w}} H(y^{tr},\hat y^{tr}_{s,\mathbf{w}})- \gamma \mathrm{KL}(\mathbb{P}(\mathbf{z}^{tr}|\hat y^{tr}_{s,\mathbf{w}})|| \mathbb{P}(\mathbf{z}^{c}|\hat y^{c}_{s,\mathbf{w}})).
    % \vspace{-0.1cm}
\end{equation}
When implementing the objective \eqref{eq:GIC with y^val} and \eqref{eq:GIC without y^val}, we further adopt a more computationally tractable form involving the difference between joint and marginal distributions to replace conditional probability distributions. Taking the case with the label $y^c$ as an example, we consider 
% \vspace{-0.1cm}
\begin{equation}
\label{eq:diff-con}
\begin{split}
    &\mathrm{KL}(\mathbb{P}(y^{tr}|\hat y^{tr}_{s,\mathbf{w}})|| \mathbb{P}(y^{c}|\hat y^{c}_{s,\mathbf{w}})) \\
    &= \mathrm{KL}(\mathbb{P}(y^{tr},\hat y^{tr}_{s,\mathbf{w}})|| \mathbb{P}(y^{c},\hat y^{c}_{s,\mathbf{w}}))- \mathrm{KL}(\mathbb{P}(\hat y^{tr}_{s,\mathbf{w}})|| \mathbb{P}(\hat y^{c}_{s,\mathbf{w}})),
\end{split}
% \vspace{-0.1cm}
\end{equation}
and then the MINE algorithm \citep{belghazi2018mutual} is employed to estimate these $\mathrm{KL}(\cdot||\cdot)$ terms. The pseudocode, outlined in Algorithm \ref{alg:GIC}, presents a comprehensive procedure for utilizing GIC to infer group labels.

\begin{figure*}
\setlength{\abovecaptionskip}{-1cm}
  \centering
  \includegraphics[width=1.\textwidth, height=0.23\textheight]{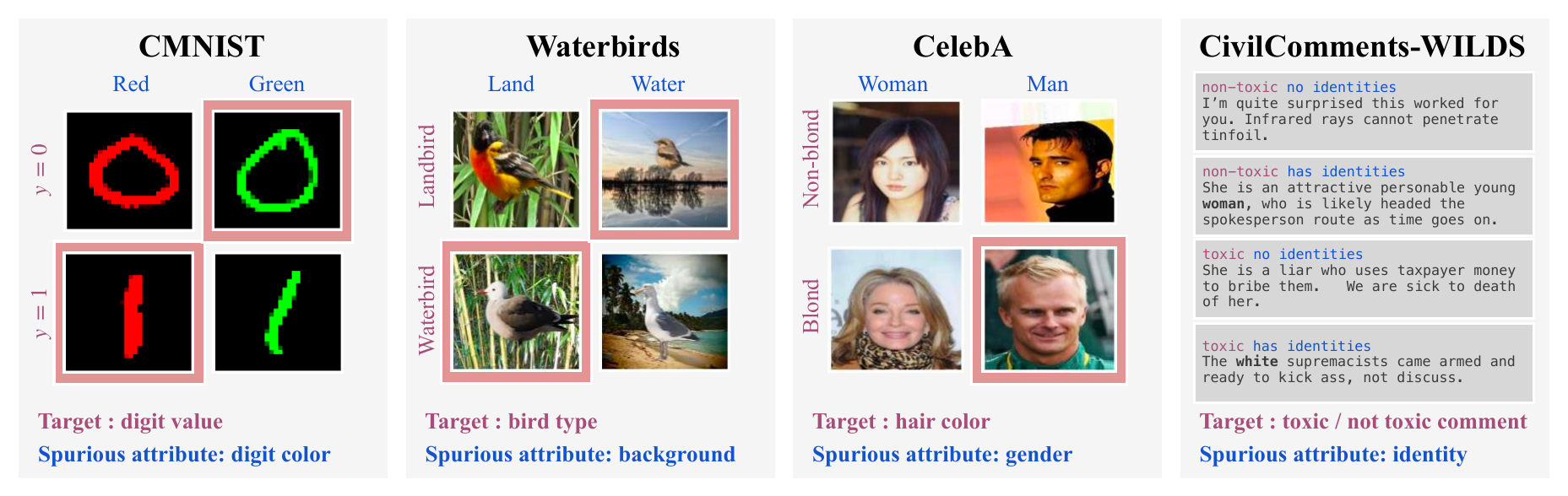}
% width=1.0\textwidth, 
% \vspace*{-10mm}
  \caption{Visualization of evaluated datasets with minority groups marked by red boxes. The spurious attribute and targets exhibit strong spurious correlations, while these correlations typically does not hold for minority groups.}
  \label{fig:dataset}
  % \vspace{-0.2cm}
\end{figure*}
\subsection{Learning Robust Classifier with GIC}
\label{Mitigating Spurious Correlation with GIC}
After obtaining the spurious attribute label $\hat{y}_{s,\mathbf{w}}$, we then infer the group label as $\hat g = (y, \hat{y}_{s,\mathbf{w}})$. The inferred groups $\hat g$ can be used with downstream invariant learning methods to learn invariant attributes and train a robust model $f_{\mathrm{robust}}$. For example, we can use Subsample technique to construct a balanced dataset by retaining all data from the smallest group and subsampling the data from the other inferred groups to match the same size for training $f_{\mathrm{robust}}$.

In \cref{Experiments}, we experiment with various downstream learning algorithms, demonstrating that the inferred groups from GIC can be flexibly utilized by diverse invariant learning algorithms to learn invariant features. This facilitates the training of robust models, denoted by $f_{\mathrm{robust}}$, which effectively mitigate spurious correlations.

\subsection{How to Obtain the Comparison Data}
\label{The Construction of Comparison Data}
Comparison data $\mathcal C$ with different group distributions from the training data is crucial in implementing GIC. The sources for comparison data are diverse, such as using a (labeled) validation data, or sampling from the unlabeled test data. In previous works, the validation set is often required to serve as a crucial basis for parameter selection \cite{liu2021just,zhang2022correct,creager2021environment}, and even directly participate in model training, aiding in group inference \cite{nam2022spread}.

In challenging scenarios where both the validation set and the test set are inaccessible, we can manually create comparison data from the training set. This can be accomplished by resampling the training dataset in a non-uniform manner, such as adjusting the sampling weight based on the trained ERM's prediction, as shown in JTT \cite{liu2021just}. Specifically, the error set (missclassification set) of a trained ERM often represents minority groups where the spurious correlation no longer exists. By sampling from the error set and non-error set, we can artificially construct comparison data with different group distributions.

Although non-uniform sampling from the training set can ensure the availability of labeled comparison data, we still emphasize the reason we consider the unlabeled comparison data, such as sampling from the test set, is that in the real world, unlabeled data is often cheaper, abundant and easier to obtain \cite{shejwalkar2023perils,gopfert2019can}. And subsequent analysis in Appendix \ref{app: The construction of comparison data} supports that abundant comparison data is beneficial for the performance of GIC. By incorporating unlabeled comparison data, the applicability of GIC is enhanced.

\subsection{The Relationship between GIC and ERM}
In this section, we highlight the relationship between GIC and ERM. When there is no difference in group distribution between training and comparison data, or the weighting parameter $\gamma = 0$, GIC degenerates to ERM in terms of group label inference.
the KL loss in \eqref{eq:GIC with y^val} and \eqref{eq:GIC without y^val} is minimal, indicating a reduced impact of the spurious term. Consequently, the CE loss becomes predominant, causing GIC to effectively become an ERM-based group inference method.

Although ERM can still be used for inferring group labels, it lacks the crucial spurious term necessary to violate the principle of invariant learning and to effectively identify spurious features. Thus, ERM-based inference should serve as the performance baseline for GIC's group label performance which is validated by subsequent experiments. Moreover, GIC can also be viewed as a special ERM-based group inference method where the spurious term functions as a regularization term. By incorporating insights from the comparison data, the spurious term encourages the trained neural network to differentiate between spurious and invariant features, thus enhancing group label inference.

\section{Experiments}
\label{Experiments}
Through our experimental evaluation, the primary goal is to address the following questions: (1) The effectiveness of GIC in mitigating spurious correlations: Can GIC successfully enhance worst-group accuracy and mitigate spurious correlation issues? (2) The accuracy of GIC in inferring group labels: Can GIC reliably predict group labels? (3) Analysis of misclassified cases: What factors contribute to misclassifications by GIC for certain instances?

Our experimental results will be presented in the order of the three questions mentioned above. In experiments, we use \(\mathrm{GIC}_{\mathcal{C}_y}\) to denote the scenario where the comparison data has true labels, and \(\mathrm{GIC}_{\mathcal{C}}\) to denote the scenario where the comparison data is unlabeled.

\begin{table*}
% \vspace{-0.2cm}
% \setlength{\abovecaptionskip}{-0.01cm}
\centering
\caption{Average and worst-group accuracy comparison (\%). Baselines are divided into two types based on whether group labels are required, and we highlight the \underline{\textbf{1st}} worst-group and the \underline{{2nd}} worst-group results for the non-group label class. $\checkmark$ denotes the use of training/validation group labels for training. GIC demonstrates strong advantages in baselines without group labels, even competing with methods with group labels on certain datasets.}
\vskip 0.1in
% \setlength{\tabcolsep}{5pt} %
% \setlength{\tabrowsep}{-2.5pt} 
% {
{
\begin{tabular}{@{}lccc|cc|cc|cc@{}}
\toprule
\multirow{2}{*}{Method} & \multirow{2}{*}{\makecell{Group Labels\\Train / Val}} & \multicolumn{2}{c}{CMNIST} &  \multicolumn{2}{c}{\makecell{Waterbirds}}& \multicolumn{2}{c}{\makecell{CelebA}} & \multicolumn{2}{c}{\makecell{CivilComments}}\\
&  & Avg.& Worst & Avg. & Worst& Avg. & Worst& Avg. & Worst\\
\midrule
% \multicolumn{1}{c}{\cellcolor{lightgray}} &\multicolumn{7}{c}{\cellcolor{lightgray}String length} \\
\multicolumn{10}{c}{\cellcolor{gray!25}\textit{Oracle Group labels are required}} \\
\midrule
GroupDRO & $\checkmark/\times$ & 74.4$\pm \scriptstyle{0.5}$ & 69.8$\pm \scriptstyle{2.6}$ &92.0$\pm \scriptstyle{0.6}$ & 89.9$\pm \scriptstyle{0.6}$&91.2$\pm \scriptstyle{0.4}$ & 87.2$\pm \scriptstyle{1.6}$&89.9$\pm \scriptstyle{0.5}$&70.0$\pm \scriptstyle{2.0}$\\
LISA& $\checkmark/\times$ & 74.0$\pm \scriptstyle{0.1}$ & 73.3$\pm \scriptstyle{0.2}$ & 91.8$\pm \scriptstyle{0.3}$ & 89.2$\pm \scriptstyle{0.6}$&92.4$\pm \scriptstyle{0.4}$ & 89.3$\pm \scriptstyle{1.1}$&89.2$\pm \scriptstyle{0.9}$ & 72.6$\pm \scriptstyle{0.1}$\\
DFR& $\times/\checkmark$ &72.2$\pm \scriptstyle{1.1}$ &70.6$\pm \scriptstyle{1.1}$ & 94.2$\pm \scriptstyle{0.4}$ & 92.9$\pm \scriptstyle{0.2}$ & 91.3$\pm \scriptstyle{0.3}$ & 88.3$\pm \scriptstyle{1.1}$&87.2$\pm \scriptstyle{0.3}$ & 70.1$\pm \scriptstyle{0.8}$\\
SSA & $\times/\checkmark$ & 75.0$\pm \scriptstyle{0.3}$ & 71.1$\pm \scriptstyle{0.4}$  &92.2$\pm \scriptstyle{0.9}$ & 89.0$\pm \scriptstyle{0.6}$&92.8$\pm \scriptstyle{0.1}$ & 89.8$\pm \scriptstyle{1.3}$&88.2$\pm \scriptstyle{2.0}$ & 69.9$\pm \scriptstyle{2.0}$\\
\midrule
\rowcolor{gray!25}\multicolumn{10}{c}{\textit{Oracle Group labels are not required}} \\
\midrule
ERM& $\times/\times$ &12.9$\pm \scriptstyle{0.8}$& 3.4$\pm \scriptstyle{0.9}$ &97.3$\pm \scriptstyle{1.0}$ & 62.6$\pm \scriptstyle{0.3}$&94.9$\pm \scriptstyle{0.3}$ & 47.7$\pm \scriptstyle{2.1}$&92.1$\pm \scriptstyle{0.4}$ & 58.6$\pm \scriptstyle{1.7}$\\
JTT& $\times/\times$ & 76.4$\pm \scriptstyle{3.3}$ & 67.3$\pm \scriptstyle{5.1} $ &89.3$\pm \scriptstyle{0.7}$ & 83.8$\pm \scriptstyle{1.2}$&88.1$\pm \scriptstyle{0.3}$& 81.5$\pm \scriptstyle{1.7}$&91.1&69.3\\
EIIL& $\times/\times$ & 74.1$\pm \scriptstyle{0.2}$ & 65.5$\pm \scriptstyle{5.1}$ &96.5$\pm \scriptstyle{0.2}$ & 77.2$\pm \scriptstyle{1.0}$&85.7$\pm \scriptstyle{0.1}$& 81.7$\pm \scriptstyle{0.8}$&90.5$\pm \scriptstyle{0.2}$ & 67.0$\pm \scriptstyle{2.4}$ \\
CnC & $\times/\times$ &  - & - &90.9$\pm \scriptstyle{0.1}$ & \underline{\textbf{88.5}}$\pm \scriptstyle{0.3}$&89.9$\pm \scriptstyle{0.5}$ & 88.8$\pm \scriptstyle{0.9}$&81.7$\pm \scriptstyle{0.5}$ & 68.9$\pm \scriptstyle{2.1}$\\
% \tablefootnote{We attempted invariant learning by contrastive learning, as mentioned in CnC, based on the oracle/GIC's inferred group labels. However, our reproduced results showed a gap from those reported by CnC. This issue was also noted by \cite{joshi2023towards}.}
\midrule
$\mathrm{GIC}_{\mathcal{C}_y}$-M & $\times/\times$  &{73.2}$\pm \scriptstyle{0.2}$&\underline{\textbf{72.2}}$\pm \scriptstyle{0.5}$ &89.6$\pm \scriptstyle{1.3}$ & \underline{{86.3}}$\pm \scriptstyle{0.1}$ &91.9$\pm \scriptstyle{0.1}$&\underline{{89.4}}$\pm \scriptstyle{0.2}$&90.0$\pm \scriptstyle{0.2}$&\underline{\textbf{72.5}}$\pm \scriptstyle{0.3}$\\
$\mathrm{GIC}_{\mathcal{C}}$-M & $\times/\times$  & 73.1$\pm \scriptstyle{0.5}$&\underline{{71.7}}$\pm \scriptstyle{0.3}$ &89.3$\pm \scriptstyle{0.8}$&  85.4$\pm \scriptstyle{0.1}$&92.1$\pm \scriptstyle{0.1}$&\underline{\textbf{89.5}}$\pm \scriptstyle{0.0}$&89.7$\pm \scriptstyle{0.0}$&\underline{{72.3}}$\pm \scriptstyle{0.2}$\\
\bottomrule
\end{tabular}}
\label{tab:acc-1}
 % \vspace{-10pt}
\end{table*}

\subsection{Experiments on Synthetic 2D Data}
\label{Synthetic 2D Data}
We start with a synthetic 2D dataset to demonstrate how GIC helps train a more robust model by learning spurious attributes. The synthetic dataset consists of training ($\mathcal{D}^{tr}$), validation ($\mathcal{D}^{val}$ which is the labeled comparison data $\mathcal{D}^{c}$), and test ($\mathcal{D}^{ts}$) sets. There is a spurious correlation in $\mathcal{D}^{tr}$, where true labels are highly correlated with the spurious attribute $\mathbf{x}_1$, while $\mathcal{D}^{c}$ and $\mathcal{D}^{ts}$ have correlations with the invariant feature $\mathbf{x}_2$. More details about the synthetic data can be found in Appendix \ref{app:Synthetic Toy Data}. \cref{fig:toy example} shows the decision boundaries of the traditional ERM model $f_\mathrm{ERM}$, the spurious attribute classifier $f_{\mathrm{GIC}}$, and the retrained robust model $f_\mathrm{robust}$ using inferred group labels from $f_{\mathrm{GIC}}$ and the Subsample strategy. We also train ERM models exclusively on the invariant feature (denoted as $f_\mathrm{invariant}$) and the spurious feature (denoted as $f_\mathrm{spurious}$).

We observe that the spurious attribute has a significant negative impact on $f_{\mathrm{ERM}}$, resulting in its decision boundary that is far from $f_{\mathrm{invariant}}$. In contrast,
the decision boundary of $f_{\mathrm{robust}}$ is much closer to $f_{\mathrm{invariant}}$, indicating its robustness achieved by leveraging the more balanced training data constructed by GIC. Furthermore, $f_{\mathrm{GIC}}$ has a decision boundary that is much closer to $f_{\mathrm{spurious}}$ compared to $f_{\mathrm{ERM}}$, indicating that GIC is better at capturing spurious attributes than ERM-based group inference methods.

\subsection{Experiments on Real-World Data}
\label{Experiments on Real-World Data}
\textbf{Real-World Datasets.} We explore datasets in image and text classification that exhibit spurious correlations. For instance, CMNIST \citep{arjovsky2019invariant} involves digit recognition with spurious features where digit colors (red or green) are linked to digit values. Waterbirds \citep{sagawa2019distributionally} associates bird types with a spurious background attribute (water or land). CelebA \citep{liu2015deep} focuses on hair color recognition influenced by spurious gender-related features. CivilComments-WILDS \cite{borkan2019nuanced,koh2021wilds} aims to distinguish toxic from non-toxic online comments, with labels spuriously correlated with mentions of demographic identities. \cref{fig:dataset} represents the spurious attributes and training targets of these datasets. Appendix \ref{app:Dataset Details} provides a detailed group distribution description of these datasets. In the main text, we primarily use labeled and unlabeled validation sets as comparison data to demonstrate the effectiveness of GIC when comparison data is directly available. The experimental evidence in Appendix \ref{app: The construction of comparison data} further illustrates the effectiveness of constructing comparison data through non-uniform sampling from the training set. 

\textbf{Baselines.} For methods that address spurious correlation while requiring group labels, we consider GroupDRO \cite{sagawa2019distributionally}, DFR \cite{kirichenko2022last}, LISA \cite{yao2022improving}, and SSA \cite{nam2022spread}. We also compare against methods that tackle spurious correlation without requiring group labels, namely, ERM, JTT \cite{liu2021just}, CnC \cite{zhang2022correct}, and EIIL \cite{creager2021environment}. ERM serves as a lower bound baseline, representing a basic training method without specific techniques to improve the accuracy of the worst-group. Additionally, we evaluate the accuracy of GIC in inferring group labels by comparing it to ERM and EI (the inferring group method of EIIL) and assess the GIC's performance when using Mixup, Subsample, Upsample, and GrouDRO as downstream invariant learning methods. Detailed information on these methods can be found in Appendix \ref{Downstream invariant learning methods}.

\textbf{Model Training.} Following stages outlined in Algorithm \ref{alg:GIC}, we provide a details of models and parameters used, particularly the selection of crucial hyperparameters training epoch $K$ and weight parameter $\gamma$ in Appendix \ref{app:Training Details}.

\textbf{Evaluation.} In order to assess how well GIC tackles spurious correlations, we report the average and worst-group accuracy for all baselines. Furthermore, to evaluate the accuracy of GIC in inferring group labels, we report the precision (proportion of correctly inferred examples belonging to the true minority group) and recall (proportion of examples from the true minority group correctly inferred) of minority groups. We focus on minority groups because their small sample size presents challenges for accurate group inference and successfully identifying these minority groups is crucial as spurious correlations often do not hold for them.

\subsection{The Effectiveness of GIC in Mitigating Spurious Correlations}
\label{Main Results}

\begin{table*}
% \vspace{-0.2cm}
% \setlength{\abovecaptionskip}{0.1cm}
\centering
\caption{Worst-group accuracy comparison (\%). We highlight the \textbf{1st} worst-group accuracy, mainly derived from GIC. The complete results including average-group accuracy are presented in Appendix \ref{Complete Results of group inference method comparison}. }
\vskip 0.1in
{
\begin{tabular}{@{}lcc|cc|cc|ccccc@{}}
\toprule
\multirow{2}{*}{Method} &
\multicolumn{2}{c}{\textit{+GroupDRO}}      & \multicolumn{2}{c}{\textit{+Subsample}} &
\multicolumn{2}{c}{\textit{+Upsample}} & \multicolumn{2}{c}{\textit{+Mixup}} \\
% &   \multicolumn{1}{c}{\makecell{Waterbirds}}& \multicolumn{1}{c}{\makecell{CelebA}}& \multicolumn{1}{c}{\makecell{Waterbirds}}& \multicolumn{1}{c}{\makecell{CelebA}}& \multicolumn{1}{c}{\makecell{Waterbirds}}& \multicolumn{1}{c}{\makecell{CelebA}}& \multicolumn{1}{c}{\makecell{Waterbirds}}& \multicolumn{1}{c}{\makecell{CelebA}}\\
% \midrule
&   Waterbirds & CelebA & Waterbirds & CelebA & Waterbirds & CelebA & Waterbirds & CelebA\\
\midrule
ERM  &75.6$\pm \scriptstyle{0.4}$&77.2$\pm \scriptstyle{0.1}$&79.4$\pm \scriptstyle{0.3}$ &78.5$\pm \scriptstyle{0.1}$&83.8$\pm \scriptstyle{1.2}$&81.5$\pm \scriptstyle{1.7}$& 82.1$\pm \scriptstyle{0.8}$ &80.6$\pm \scriptstyle{1.7}$\\
EI& 77.2$\pm \scriptstyle{1.0}$&81.7$\pm \scriptstyle{0.8}$&81.9$\pm \scriptstyle{1.4}$ &82.8$\pm \scriptstyle{0.5}$&81.3$\pm \scriptstyle{0.7}$&84.8$\pm \scriptstyle{0.2}$&85.7$\pm \scriptstyle{0.4}$ &84.9$\pm \scriptstyle{3.7}$\\
$\mathrm{GIC}_{\mathcal{C}_y}$&{\textbf{80.2}}$\pm \scriptstyle{0.1}$& {\textbf{82.1}}$\pm \scriptstyle{0.3}$&{\textbf{83.5}}$\pm \scriptstyle{0.8}$& {\textbf{86.1}}$\pm \scriptstyle{2.2}$ & {\textbf{84.1}}$\pm \scriptstyle{0.0}$&87.2$\pm \scriptstyle{0.0}$& {\textbf{86.3}}$\pm \scriptstyle{0.1}$ &89.4$\pm \scriptstyle{0.2}$\\
$\mathrm{GIC}_{\mathcal{C}}$&79.2$\pm \scriptstyle{0.4}$& 79.7$\pm \scriptstyle{0.6}$ & 82.1$\pm \scriptstyle{1.1}$& 83.1$\pm \scriptstyle{0.3}$&82.1$\pm \scriptstyle{0.7}$&{\textbf{87.8}}$\pm \scriptstyle{1.1}$&85.4$\pm \scriptstyle{0.1}$&{\textbf{89.5}}$\pm \scriptstyle{0.0}$\\
\bottomrule
\end{tabular}}
\label{tab:acc-2}
 % \vspace{-10pt}
\end{table*}

\cref{tab:acc-1} showcases the average and worst-group accuracies for all methods. We specifically highlight GIC's performance when combined with Mixup (denoted as $\mathrm{GIC}_{\mathcal{C}_y}$-M and $\mathrm{GIC}_{\mathcal{C}}$-M). In comparison to methods that train without leveraging group information, $\mathrm{GIC}_{\mathcal{C}_y}$-M and $\mathrm{GIC}_{\mathcal{C}}$-M consistently achieve higher worst-group accuracy across all datasets. Notably, even when compared to methods that incorporate group information during training, $\mathrm{GIC}_{\mathcal{C}_y}$-M and $\mathrm{GIC}_{\mathcal{C}}$-M deliver impressive results, particularly on CelebA and CivilComments datasets, where $\mathrm{GIC}_{\mathcal{C}_y}$-M almost matches the performance of baselines that utilize group labels. Furthermore, we observe that methods without group labels demonstrate weaker performance on worst-group accuracy compared to oracle group label-based methods. This discrepancy is especially pronounced when employing the same invariant learning algorithm (e.g., GroupDRO and EIIL, $\mathrm{GIC}_{\mathcal{C}_y}$-M and LISA). It underscores the necessity of further enhancing the accuracy of inferred group labels to boost the worst-group accuracy. The baselines, including JTT, CNC, EIIL, and SSA, tune hyperparameters and employ early stopping based on the highest worst-group accuracy observed on the validation set. Similarly, we also consider using the highest worst-group accuracy of the validation set with group labels as a criterion to ascertain the optimal number of training epochs. The ablation studies detailed in Appendix \ref{Ablation Study} highlight the importance of early stopping as a strategy in the GIC framework.

\begin{figure} 
% \vspace{-0.8cm} 
% \setlength{\abovecaptionskip}{-3cm}
  \centering
  \includegraphics[width=.45\textwidth, height=0.25\textheight]{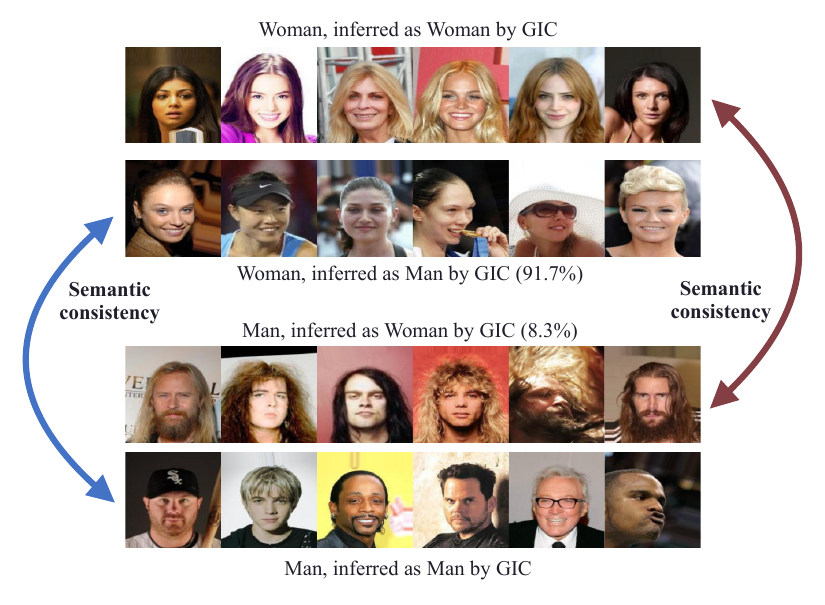}
  % \captionsetup{skip=-1pt}
  % \vspace*{-5mm}
  \caption{Misclassified samples on CelebA. The semantic consistency in GIC leads to the misclassification of women with short hair (a typical characteristic of males) as males (91.7\%), and men with long hair (a typical characteristic of females) as females (8.3\%).}
  \label{fig:Error Case on CelebA}
  % \vspace{-0.8cm} 
  % \vspace{-30pt}
\end{figure}

\subsection{The Accuracy of GIC in Inferring
Group Labels}
\label{The scalability of GIC}

We then delve deeper into comparing the performance of GIC, ERM, and EI in inferring group labels of minority groups. In \cref{fig:sacc}, we find that GIC exhibits higher precision compared to the baselines on almost all datasets. This higher precision suggests that the minority groups estimated by GIC have low redundancy, thereby increasing the possibility of utilizing true minority group examples to assist in training robust models for downstream tasks. Furthermore, GIC's primary advantage lies in its high recall, consistently maintained at over 60\% and even surpassing 80\% in CMNIST, CelebA, and CivilComments. The high recall rate indicates the estimated minority groups by GIC contain more diverse samples from true minority groups which can provide more discriminative information for downstream invariant learning tasks \cite{gong2019diversity}.

In \cref{tab:acc-2}, we present our findings on combining GIC, ERM, and EI with various invariant learning algorithms, including GroupDRO \cite{duchi2019distributionally}, Subsample \cite{kirichenko2022last}, and Upsample \cite{liu2021just}. Our worst-group results show that GIC consistently outperforms the baselines across different downstream algorithms, further emphasizing its superiority in group inference and worst-group performance improvement. The ablation results concerning the use of the early stopping strategy when training robust models are also shown in the Appendix \ref{Ablation Study}.

\subsection{Error Case Analysis}
\label{Error Case}
In this section, we then focus on the underlying factors contributing to these errors via visualizing the misclassified samples. \cref{fig:Error Case on CelebA} shows examples from the CelebA dataset, where the spurious attribute is gender (woman and man), and the association between hair color and gender is considered spurious, such as blonde hair is correlated with women. We note an interesting phenomenon known as semantic consistency in GIC. For instance, GIC misclassifies women with short hair as men, who bear a strong semantic resemblance to correctly classified male samples with short hair. This phenomenon is also evident in the waterbirds dataset, as detailed in Appendix \ref{Error Case}, where GIC tends to classify water backgrounds with prominent land elements (such as trees) as land backgrounds. This semantic consistency significantly contributes to prediction errors in GIC and has a negative effect on downstream algorithms that rely on sampling. It can lead to an overrepresentation of majority group samples ([blond, women]) within the estimated minority group ([blond, men]), exacerbating group imbalances during resampling.

 % or classify land backgrounds with expansive blue skies as water backgrounds
 % Similarly, GIC tends to misclassify men with long hair—a feature typically found in female images—as women.

\begin{figure}[]
  \centering
  \includegraphics[width=0.48\textwidth, height=0.2\textheight]{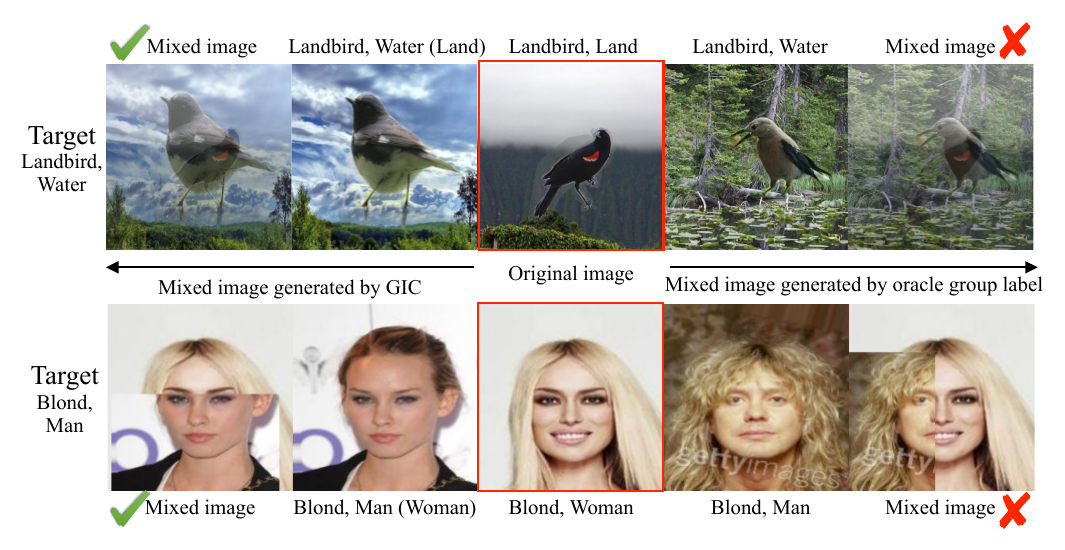}
% width=1.0\textwidth, 
% \vspace*{-10mm}
  \caption{GIC generates better mixed images. By leveraging the high semantic consistency in image recognition, spurious attributes and true labels are decoupled in mixed images generated by GIC. Various mixing techniques are employed to handle different datasets, as detailed in \cref{Downstream invariant learning methods}.}
  \label{fig:erorr case-celeba-mixup}
    % \vspace{-0.5cm} 
\end{figure}

However, for invariant methods like Mixup \cite{yao2022improving} that aim to disrupt spurious correlations between spurious attributes and true labels for invariant learning, high semantic consistency is beneficial. We then show how GIC's semantic consistency leads to better mixed images compared to using oracle group labels. In \cref{fig:erorr case-celeba-mixup}, Mixup disrupts the spurious correlation between blonde hair and women by mixing samples from the same class (blond) but with different spurious attributes (man). When sampling from the oracle blond man group, long-haired men may be selected, resulting in mixed images that still retain the typical woman attribute (long hair). However, by using the blond man sample inferred by GIC, such as the short-haired woman in \cref{fig:erorr case-celeba-mixup}, the generated mixed samples is more closely resemble blonde men compared to using oracle group labels. Similarly, advantages can also occur in the waterbirds dataset as shown in \cref{fig:erorr case-celeba-mixup}. These findings may explain the better performance of $\mathrm{GIC}_{\mathcal{C}_y}$-M than LISA (which uses oracle group labels with Mixup) in \cref{tab:acc-1}, even though GIC's precision and recall are not at 100\%. 

\section{Discussion}
\label{Discussion}
In this work, we introduce GIC, a novel method for better inferring group labels and mitigating spurious correlations without any additional information requirement. Experiments on synthetic and real data reveal the enhancement of GIC in group label inference and its flexibility in integration with various invariant learning algorithms. Interestingly, our analysis of misclassification cases reveals the semantic consistency phenomenon in GIC, which effectively disrupts spurious correlations and facilitates invariant learning.

Comparison data is essential for GIC and can be the labeled validation set, the subset of the unlabeled test set, or constructed non-uniformly from the training set, making GIC universally applicable. Additional experiments in Appendix \ref{app: The construction of comparison data} demonstrate the feasibility of constructing comparison data non-uniformly from the training set, achieving comparable performance on worst-group accuracy as directly using an validation set as comparison data. Moreover, Appendix \ref{app: The construction of comparison data} further emphasizes that while non-uniform sampling from training data can construct comparison data, considering larger and cheaper unlabeled data as comparison data is necessary since an increased sample size of comparison data improves GIC's performance. We also emphasize that the group distribution difference between comparison and training data can be subtle. \cref{Main Results} shows GIC's effectiveness on CelebA and Civilcomments, despite the high similarity in group distributions between the training and comparison data. In Appendix \ref{app:distribution diff}, we further investigate the positive impact of group differences on GIC's performance. We propose readjusting the comparison data's group distribution based on GIC's inferred groups to further enhance the difference in group distribution and improve worst-group performance. This approach is supported by experiments detailed in Appendix \ref{app:distribution diff}.

\newpage
\section*{Acknowledgements}
We would like to thank the anonymous reviewers and area chairs for their helpful comments. YH and DZ are supported by NSFC 62306252 and the central fund from HKU IDS.
% \textbf{Do not} include acknowledgements in the initial version of
% the paper submitted for blind review.

% If a paper is accepted, the final camera-ready version can (and
% usually should) include acknowledgements.  Such acknowledgements
% should be placed at the end of the section, in an unnumbered section
% that does not count towards the paper page limit. Typically, this will 
% include thanks to reviewers who gave useful comments, to colleagues 
% who contributed to the ideas, and to funding agencies and corporate 
% sponsors that provided financial support.

\section*{Impact Statement}

% Authors are \textbf{required} to include a statement of the potential 
% broader impact of their work, including its ethical aspects and future 
% societal consequences. This statement should be in an unnumbered 
% section at the end of the paper (co-located with Acknowledgements -- 
% the two may appear in either order, but both must be before References), 
% and does not count toward the paper page limit. In many cases, where 
% the ethical impacts and expected societal implications are those that 
% are well established when advancing the field of Machine Learning, 
% substantial discussion is not required, and a simple statement such 
% as the following will suffice:

This paper presents work whose goal is to advance the field of 
Machine Learning. There are many potential societal consequences 
of our work, none which we feel must be specifically highlighted here.

% The above statement can be used verbatim in such cases, but we 
% encourage authors to think about whether there is content which does 
% warrant further discussion, as this statement will be apparent if the 
% paper is later flagged for ethics review.
% \newpage
\nocite{langley00}
\bibliography{GIC_CameraReady}

\begin{thebibliography}{52}
\providecommand{\natexlab}[1]{#1}
\providecommand{\url}[1]{\texttt{#1}}
\expandafter\ifx\csname urlstyle\endcsname\relax
  \providecommand{\doi}[1]{doi: #1}\else
  \providecommand{\doi}{doi: \begingroup \urlstyle{rm}\Url}\fi

\bibitem[Ahmed et~al.(2020)Ahmed, Bengio, Van~Seijen, and Courville]{ahmed2020systematic}
Ahmed, F., Bengio, Y., Van~Seijen, H., and Courville, A.
\newblock Systematic generalisation with group invariant predictions.
\newblock In \emph{International Conference on Learning Representations}, 2020.

\bibitem[Arjovsky et~al.(2019)Arjovsky, Bottou, Gulrajani, and Lopez-Paz]{arjovsky2019invariant}
Arjovsky, M., Bottou, L., Gulrajani, I., and Lopez-Paz, D.
\newblock Invariant risk minimization.
\newblock \emph{arXiv preprint arXiv:1907.02893}, 2019.

\bibitem[Badgeley et~al.(2019)Badgeley, Zech, Oakden-Rayner, Glicksberg, Liu, Gale, McConnell, Percha, Snyder, and Dudley]{badgeley2019deep}
Badgeley, M.~A., Zech, J.~R., Oakden-Rayner, L., Glicksberg, B.~S., Liu, M., Gale, W., McConnell, M.~V., Percha, B., Snyder, T.~M., and Dudley, J.~T.
\newblock Deep learning predicts hip fracture using confounding patient and healthcare variables.
\newblock \emph{NPJ digital medicine}, 2\penalty0 (1):\penalty0 31, 2019.

\bibitem[Beery et~al.(2018)Beery, Van~Horn, and Perona]{beery2018recognition}
Beery, S., Van~Horn, G., and Perona, P.
\newblock Recognition in terra incognita.
\newblock In \emph{Proceedings of the European conference on computer vision (ECCV)}, pp.\  456--473, 2018.

\bibitem[Belghazi et~al.(2018)Belghazi, Baratin, Rajeshwar, Ozair, Bengio, Courville, and Hjelm]{belghazi2018mutual}
Belghazi, M.~I., Baratin, A., Rajeshwar, S., Ozair, S., Bengio, Y., Courville, A., and Hjelm, D.
\newblock Mutual information neural estimation.
\newblock In \emph{International conference on machine learning}, pp.\  531--540. PMLR, 2018.

\bibitem[Blanchard et~al.(2011)Blanchard, Lee, and Scott]{blanchard2011generalizing}
Blanchard, G., Lee, G., and Scott, C.
\newblock Generalizing from several related classification tasks to a new unlabeled sample.
\newblock \emph{Advances in neural information processing systems}, 24, 2011.

\bibitem[Blodgett et~al.(2016)Blodgett, Green, and O'Connor]{blodgett2016demographic}
Blodgett, S.~L., Green, L., and O'Connor, B.
\newblock Demographic dialectal variation in social media: A case study of african-american english.
\newblock \emph{arXiv preprint arXiv:1608.08868}, 2016.

\bibitem[Borkan et~al.(2019)Borkan, Dixon, Sorensen, Thain, and Vasserman]{borkan2019nuanced}
Borkan, D., Dixon, L., Sorensen, J., Thain, N., and Vasserman, L.
\newblock Nuanced metrics for measuring unintended bias with real data for text classification.
\newblock In \emph{Companion proceedings of the 2019 world wide web conference}, pp.\  491--500, 2019.

\bibitem[Cao et~al.(2019)Cao, Wei, Gaidon, Arechiga, and Ma]{cao2019learning}
Cao, K., Wei, C., Gaidon, A., Arechiga, N., and Ma, T.
\newblock Learning imbalanced datasets with label-distribution-aware margin loss.
\newblock \emph{Advances in neural information processing systems}, 32, 2019.

\bibitem[Chen et~al.(2023)Chen, Huang, Zhou, Bian, Han, and Cheng]{chen2023understanding}
Chen, Y., Huang, W., Zhou, K., Bian, Y., Han, B., and Cheng, J.
\newblock Understanding and improving feature learning for out-of-distribution generalization.
\newblock In \emph{Thirty-seventh Conference on Neural Information Processing Systems}, 2023.

\bibitem[Cover(1999)]{cover1999elements}
Cover, T.~M.
\newblock \emph{Elements of information theory}.
\newblock John Wiley \& Sons, 1999.

\bibitem[Creager et~al.(2021)Creager, Jacobsen, and Zemel]{creager2021environment}
Creager, E., Jacobsen, J.-H., and Zemel, R.
\newblock Environment inference for invariant learning.
\newblock In \emph{International Conference on Machine Learning}, pp.\  2189--2200. PMLR, 2021.

\bibitem[Duchi et~al.(2019)Duchi, Hashimoto, and Namkoong]{duchi2019distributionally}
Duchi, J.~C., Hashimoto, T., and Namkoong, H.
\newblock Distributionally robust losses against mixture covariate shifts.
\newblock \emph{Under review}, 2\penalty0 (1), 2019.

\bibitem[Ganin \& Lempitsky(2015)Ganin and Lempitsky]{ganin2015unsupervised}
Ganin, Y. and Lempitsky, V.
\newblock Unsupervised domain adaptation by backpropagation.
\newblock In \emph{International conference on machine learning}, pp.\  1180--1189. PMLR, 2015.

\bibitem[Ganin et~al.(2016)Ganin, Ustinova, Ajakan, Germain, Larochelle, Laviolette, March, and Lempitsky]{ganin2016domain}
Ganin, Y., Ustinova, E., Ajakan, H., Germain, P., Larochelle, H., Laviolette, F., March, M., and Lempitsky, V.
\newblock Domain-adversarial training of neural networks.
\newblock \emph{Journal of machine learning research}, 17\penalty0 (59):\penalty0 1--35, 2016.

\bibitem[Geirhos et~al.(2020)Geirhos, Jacobsen, Michaelis, Zemel, Brendel, Bethge, and Wichmann]{geirhos2020shortcut}
Geirhos, R., Jacobsen, J.-H., Michaelis, C., Zemel, R., Brendel, W., Bethge, M., and Wichmann, F.~A.
\newblock Shortcut learning in deep neural networks.
\newblock \emph{Nature Machine Intelligence}, 2\penalty0 (11):\penalty0 665--673, 2020.

\bibitem[Gong et~al.(2019)Gong, Zhong, and Hu]{gong2019diversity}
Gong, Z., Zhong, P., and Hu, W.
\newblock Diversity in machine learning.
\newblock \emph{Ieee Access}, 7:\penalty0 64323--64350, 2019.

\bibitem[G{\"o}pfert et~al.(2019)G{\"o}pfert, Ben-David, Bousquet, Gelly, Tolstikhin, and Urner]{gopfert2019can}
G{\"o}pfert, C., Ben-David, S., Bousquet, O., Gelly, S., Tolstikhin, I., and Urner, R.
\newblock When can unlabeled data improve the learning rate?
\newblock In \emph{Conference on Learning Theory}, pp.\  1500--1518. PMLR, 2019.

\bibitem[Gururangan et~al.(2018)Gururangan, Swayamdipta, Levy, Schwartz, Bowman, and Smith]{gururangan2018annotation}
Gururangan, S., Swayamdipta, S., Levy, O., Schwartz, R., Bowman, S.~R., and Smith, N.~A.
\newblock Annotation artifacts in natural language inference data.
\newblock \emph{arXiv preprint arXiv:1803.02324}, 2018.

\bibitem[Hashimoto et~al.(2018)Hashimoto, Srivastava, Namkoong, and Liang]{hashimoto2018fairness}
Hashimoto, T., Srivastava, M., Namkoong, H., and Liang, P.
\newblock Fairness without demographics in repeated loss minimization.
\newblock In \emph{International Conference on Machine Learning}, pp.\  1929--1938. PMLR, 2018.

\bibitem[Izmailov et~al.(2022)Izmailov, Kirichenko, Gruver, and Wilson]{izmailov2022feature}
Izmailov, P., Kirichenko, P., Gruver, N., and Wilson, A.~G.
\newblock On feature learning in the presence of spurious correlations.
\newblock \emph{Advances in Neural Information Processing Systems}, 35:\penalty0 38516--38532, 2022.

\bibitem[Kirichenko et~al.(2022)Kirichenko, Izmailov, and Wilson]{kirichenko2022last}
Kirichenko, P., Izmailov, P., and Wilson, A.~G.
\newblock Last layer re-training is sufficient for robustness to spurious correlations.
\newblock \emph{arXiv preprint arXiv:2204.02937}, 2022.

\bibitem[Koh et~al.(2021)Koh, Sagawa, Marklund, Xie, Zhang, Balsubramani, Hu, Yasunaga, Phillips, Gao, et~al.]{koh2021wilds}
Koh, P.~W., Sagawa, S., Marklund, H., Xie, S.~M., Zhang, M., Balsubramani, A., Hu, W., Yasunaga, M., Phillips, R.~L., Gao, I., et~al.
\newblock Wilds: A benchmark of in-the-wild distribution shifts.
\newblock In \emph{International Conference on Machine Learning}, pp.\  5637--5664. PMLR, 2021.

\bibitem[Kong et~al.(2019)Kong, d'Autume, Ling, Yu, Dai, and Yogatama]{kong2019mutual}
Kong, L., d'Autume, C. d.~M., Ling, W., Yu, L., Dai, Z., and Yogatama, D.
\newblock A mutual information maximization perspective of language representation learning.
\newblock \emph{arXiv preprint arXiv:1910.08350}, 2019.

\bibitem[Lagnado \& Sloman(2019)Lagnado and Sloman]{lagnado2019learning}
Lagnado, D.~A. and Sloman, S.
\newblock Learning causal structure.
\newblock In \emph{Proceedings of the twenty-fourth annual conference of the Cognitive Science Society}, pp.\  560--565. Routledge, 2019.

\bibitem[Langley(2000)]{langley00}
Langley, P.
\newblock Crafting papers on machine learning.
\newblock In Langley, P. (ed.), \emph{Proceedings of the 17th International Conference on Machine Learning (ICML 2000)}, pp.\  1207--1216, Stanford, CA, 2000. Morgan Kaufmann.

\bibitem[Lee et~al.(2022)Lee, Yao, and Finn]{lee2022diversify}
Lee, Y., Yao, H., and Finn, C.
\newblock Diversify and disambiguate: Out-of-distribution robustness via disagreement.
\newblock In \emph{The Eleventh International Conference on Learning Representations}, 2022.

\bibitem[Li et~al.(2022)Li, Shen, Wang, Zhu, Li, Keutzer, and Zhao]{li2022invariant}
Li, B., Shen, Y., Wang, Y., Zhu, W., Li, D., Keutzer, K., and Zhao, H.
\newblock Invariant information bottleneck for domain generalization.
\newblock In \emph{Proceedings of the AAAI Conference on Artificial Intelligence}, volume~36, pp.\  7399--7407, 2022.

\bibitem[Li \& Jurafsky(2016)Li and Jurafsky]{li2016mutual}
Li, J. and Jurafsky, D.
\newblock Mutual information and diverse decoding improve neural machine translation.
\newblock \emph{arXiv preprint arXiv:1601.00372}, 2016.

\bibitem[Lin et~al.(2022)Lin, Zhu, Tan, and Cui]{lin2022zin}
Lin, Y., Zhu, S., Tan, L., and Cui, P.
\newblock Zin: When and how to learn invariance without environment partition?
\newblock \emph{Advances in Neural Information Processing Systems}, 35:\penalty0 24529--24542, 2022.

\bibitem[Liu et~al.(2021)Liu, Haghgoo, Chen, Raghunathan, Koh, Sagawa, Liang, and Finn]{liu2021just}
Liu, E.~Z., Haghgoo, B., Chen, A.~S., Raghunathan, A., Koh, P.~W., Sagawa, S., Liang, P., and Finn, C.
\newblock Just train twice: Improving group robustness without training group information.
\newblock In \emph{International Conference on Machine Learning}, pp.\  6781--6792. PMLR, 2021.

\bibitem[Liu et~al.(2015)Liu, Luo, Wang, and Tang]{liu2015deep}
Liu, Z., Luo, P., Wang, X., and Tang, X.
\newblock Deep learning face attributes in the wild.
\newblock In \emph{Proceedings of the IEEE international conference on computer vision}, pp.\  3730--3738, 2015.

\bibitem[Muandet et~al.(2013)Muandet, Balduzzi, and Sch{\"o}lkopf]{muandet2013domain}
Muandet, K., Balduzzi, D., and Sch{\"o}lkopf, B.
\newblock Domain generalization via invariant feature representation.
\newblock In \emph{International conference on machine learning}, pp.\  10--18. PMLR, 2013.

\bibitem[Nam et~al.(2020)Nam, Cha, Ahn, Lee, and Shin]{nam2020learning}
Nam, J., Cha, H., Ahn, S., Lee, J., and Shin, J.
\newblock Learning from failure: De-biasing classifier from biased classifier.
\newblock \emph{Advances in Neural Information Processing Systems}, 33:\penalty0 20673--20684, 2020.

\bibitem[Nam et~al.(2022)Nam, Kim, Lee, and Shin]{nam2022spread}
Nam, J., Kim, J., Lee, J., and Shin, J.
\newblock Spread spurious attribute: Improving worst-group accuracy with spurious attribute estimation.
\newblock \emph{arXiv preprint arXiv:2204.02070}, 2022.

\bibitem[Pan et~al.(2020)Pan, Yang, Liang, Kailkhura, Jin, Hua, Cai, and Li]{pan2020adversarial}
Pan, B., Yang, Y., Liang, K., Kailkhura, B., Jin, Z., Hua, X.-S., Cai, D., and Li, B.
\newblock Adversarial mutual information for text generation.
\newblock In \emph{International Conference on Machine Learning}, pp.\  7476--7486. PMLR, 2020.

\bibitem[Paninski(2003)]{paninski2003estimation}
Paninski, L.
\newblock Estimation of entropy and mutual information.
\newblock \emph{Neural computation}, 15\penalty0 (6):\penalty0 1191--1253, 2003.

\bibitem[Ribeiro et~al.(2016)Ribeiro, Singh, and Guestrin]{ribeiro2016should}
Ribeiro, M.~T., Singh, S., and Guestrin, C.
\newblock " why should i trust you?" explaining the predictions of any classifier.
\newblock In \emph{Proceedings of the 22nd ACM SIGKDD international conference on knowledge discovery and data mining}, pp.\  1135--1144, 2016.

\bibitem[Rosenfeld et~al.(2022)Rosenfeld, Ravikumar, and Risteski]{rosenfeld2022domain}
Rosenfeld, E., Ravikumar, P., and Risteski, A.
\newblock Domain-adjusted regression or: Erm may already learn features sufficient for out-of-distribution generalization.
\newblock \emph{arXiv preprint arXiv:2202.06856}, 2022.

\bibitem[Sagawa et~al.(2019)Sagawa, Koh, Hashimoto, and Liang]{sagawa2019distributionally}
Sagawa, S., Koh, P.~W., Hashimoto, T.~B., and Liang, P.
\newblock Distributionally robust neural networks for group shifts: On the importance of regularization for worst-case generalization.
\newblock \emph{arXiv preprint arXiv:1911.08731}, 2019.

\bibitem[Sagawa et~al.(2020)Sagawa, Raghunathan, Koh, and Liang]{sagawa2020investigation}
Sagawa, S., Raghunathan, A., Koh, P.~W., and Liang, P.
\newblock An investigation of why overparameterization exacerbates spurious correlations.
\newblock In \emph{International Conference on Machine Learning}, pp.\  8346--8356. PMLR, 2020.

\bibitem[Shejwalkar et~al.(2023)Shejwalkar, Lyu, and Houmansadr]{shejwalkar2023perils}
Shejwalkar, V., Lyu, L., and Houmansadr, A.
\newblock The perils of learning from unlabeled data: Backdoor attacks on semi-supervised learning.
\newblock In \emph{Proceedings of the IEEE/CVF International Conference on Computer Vision}, pp.\  4730--4740, 2023.

\bibitem[Sohoni et~al.(2020)Sohoni, Dunnmon, Angus, Gu, and R{\'e}]{sohoni2020no}
Sohoni, N., Dunnmon, J., Angus, G., Gu, A., and R{\'e}, C.
\newblock No subclass left behind: Fine-grained robustness in coarse-grained classification problems.
\newblock \emph{Advances in Neural Information Processing Systems}, 33:\penalty0 19339--19352, 2020.

\bibitem[Sohoni et~al.(2021)Sohoni, Sanjabi, Ballas, Grover, Nie, Firooz, and R{\'e}]{sohoni2021barack}
Sohoni, N.~S., Sanjabi, M., Ballas, N., Grover, A., Nie, S., Firooz, H., and R{\'e}, C.
\newblock Barack: Partially supervised group robustness with guarantees.
\newblock \emph{arXiv preprint arXiv:2201.00072}, 2021.

\bibitem[Su et~al.(2023)Su, Zhu, Tao, Lu, Li, Huang, Qiao, Wang, Zhou, and Dai]{su2023towards}
Su, W., Zhu, X., Tao, C., Lu, L., Li, B., Huang, G., Qiao, Y., Wang, X., Zhou, J., and Dai, J.
\newblock Towards all-in-one pre-training via maximizing multi-modal mutual information.
\newblock In \emph{Proceedings of the IEEE/CVF Conference on Computer Vision and Pattern Recognition}, pp.\  15888--15899, 2023.

\bibitem[Tatman(2017)]{tatman2017gender}
Tatman, R.
\newblock Gender and dialect bias in youtube’s automatic captions.
\newblock In \emph{Proceedings of the first ACL workshop on ethics in natural language processing}, pp.\  53--59, 2017.

\bibitem[Wah et~al.(2011)Wah, Branson, Welinder, Perona, and Belongie]{wah2011caltech}
Wah, C., Branson, S., Welinder, P., Perona, P., and Belongie, S.
\newblock The caltech-ucsd birds-200-2011 dataset.
\newblock 2011.

\bibitem[Wu et~al.(2023)Wu, Yuksekgonul, Zhang, and Zou]{wu2023discover}
Wu, S., Yuksekgonul, M., Zhang, L., and Zou, J.
\newblock Discover and cure: Concept-aware mitigation of spurious correlation.
\newblock \emph{arXiv preprint arXiv:2305.00650}, 2023.

\bibitem[Yang et~al.(2023)Yang, Zhang, Katabi, and Ghassemi]{yang2023change}
Yang, Y., Zhang, H., Katabi, D., and Ghassemi, M.
\newblock Change is hard: A closer look at subpopulation shift.
\newblock \emph{arXiv preprint arXiv:2302.12254}, 2023.

\bibitem[Yao et~al.(2022)Yao, Wang, Li, Zhang, Liang, Zou, and Finn]{yao2022improving}
Yao, H., Wang, Y., Li, S., Zhang, L., Liang, W., Zou, J., and Finn, C.
\newblock Improving out-of-distribution robustness via selective augmentation.
\newblock In \emph{International Conference on Machine Learning}, pp.\  25407--25437. PMLR, 2022.

\bibitem[Zhang et~al.(2022)Zhang, Sohoni, Zhang, Finn, and R{\'e}]{zhang2022correct}
Zhang, M., Sohoni, N.~S., Zhang, H.~R., Finn, C., and R{\'e}, C.
\newblock Correct-n-contrast: A contrastive approach for improving robustness to spurious correlations.
\newblock \emph{arXiv preprint arXiv:2203.01517}, 2022.

\bibitem[Zhou et~al.(2017)Zhou, Lapedriza, Khosla, Oliva, and Torralba]{zhou2017places}
Zhou, B., Lapedriza, A., Khosla, A., Oliva, A., and Torralba, A.
\newblock Places: A 10 million image database for scene recognition.
\newblock \emph{IEEE transactions on pattern analysis and machine intelligence}, 40\penalty0 (6):\penalty0 1452--1464, 2017.

\end{thebibliography}
\bibliographystyle{icml2024}

%%%%%%%%%%%%%%%%%%%%%%%%%%%%%%%%%%%%%%%%%%%%%%%%%%%%%%%%%%%%%%%%%%%%%%%%%%%%%%%
%%%%%%%%%%%%%%%%%%%%%%%%%%%%%%%%%%%%%%%%%%%%%%%%%%%%%%%%%%%%%%%%%%%%%%%%%%%%%%%
% APPENDIX
%%%%%%%%%%%%%%%%%%%%%%%%%%%%%%%%%%%%%%%%%%%%%%%%%%%%%%%%%%%%%%%%%%%%%%%%%%%%%%%
%%%%%%%%%%%%%%%%%%%%%%%%%%%%%%%%%%%%%%%%%%%%%%%%%%%%%%%%%%%%%%%%%%%%%%%%%%%%%%%
\newpage
\appendix
\onecolumn
\section{Proof Details}
\label{app:Proof Details} 
Given $y^{tr}$, the maximization of mutual information $I(y^{tr};\hat y^{tr}_{s,\mathbf{w}})$ can be transformed into the minimization of cross-entropy $H(y^{tr},\hat y^{tr}_{s,\mathbf{w}})$.

To demonstrate this assertion, we establish the following lemma:

\begin{lemma}
\label{proof of Lower Bound of Correlation Term}
    % \textbf{(Restatement of Lemma \ref{Lower Bound of Correlation Term})} 
    Given $y^{tr}$, the correlation term is lower bounded by the difference between the entropy of $y^{tr}$ and the cross-entropy between $y^{tr}$ and $\hat y^{tr}_{s,\mathbf{w}}$:
    \begin{equation*}
    I(y^{tr};\hat y^{tr}_{s,\mathbf{w}}) \geq H(y^{tr})-H(y^{tr},\hat y^{tr}_{s,\mathbf{w}}).
    \end{equation*}
\end{lemma}

\begin{proof}
For clearer expression, we omit the subscript $\mathbf{w}$ representing the training parameters in $\hat y_{s,\mathbf{w}}$.
We first expand the mutual information term:
\begin{equation*}
 \begin{split}
    I(y^{tr};\hat y^{tr}_{s}) &= H(y^{tr})-H(y^{tr}|\hat y^{tr}_{s})\\
    \end{split}
    \end{equation*}
Since, the cross-entropy can be expanded as:
\begin{equation*}
\begin{split}
    H(y^{tr},\hat y^{tr}_{s}) &= -\sum \mathbb{P}(y^{tr})\log \mathbb{P}(\hat y^{tr}_{s}) \\
    &= \sum \mathbb{P}(y^{tr})\log \frac{\mathbb{P}(y^{tr})}{\mathbb{P}(\hat y^{tr}_{s})} - \sum \mathbb{P}(y^{tr})\log \mathbb{P}(y^{tr}) \\
    &= \mathrm{KL}(y^{tr}||\hat y^{tr}_{s}) + H(y^{tr}) \\
    &\geq \mathrm{KL}(y^{tr}||\hat y^{tr}_{s}) + H(y^{tr}|\hat y^{tr}_{s})\\
    &\geq H(y^{tr}|\hat y^{tr}_{s}).
\end{split}
\end{equation*}
Then we have,
\begin{equation*}
 \begin{split}
    I(y^{tr};\hat y^{tr}_{s}) &= H(y^{tr})-H(y|\hat y^{tr}_{s})\geq H(y^{tr})-H(y^{tr},\hat y^{tr}_{s}).
    \end{split}
    \end{equation*}
According to \cref{proof of Lower Bound of Correlation Term}, when $y^{tr}$ is given, maximizing the mutual information $I(y^{tr};\hat y^{tr}_{s,\mathbf{w}})$ can be transformed into minimizing $H(y^{tr},\hat y^{tr}_{s,\mathbf{w}})$, which is more computationally tractable in practice. 
\end{proof}
\subsection{Proof of Theorem \ref{Lower Bound of Spurious Term without Accessible Y}}
\label{app:proof for theorem 1}

To prove Theorem \ref{Lower Bound of Spurious Term without Accessible Y}, we first demonstrate the causal structure between $y_s$, $\mathbf{z}$, and $y$. Following the work \cite{li2022invariant}, we assume $y_s \leftarrow \mathbf{z} \rightarrow y$. The fork causal structure \citep{lagnado2019learning} between $y_s$ , $\mathbf{z}$ and $y$ exhibits the following properties:

\begin{property}
\label{pro:fork causal}
The fork causal relationship $y_s \leftarrow \mathbf{z} \rightarrow y$ adheres to the following properties:
\begin{enumerate}
    \item  $y_s \not\perp y$ means the true label $y$ and the spurious label $y_s$ are dependent. %Thus, $\mathbb{P}(y_s y) \neq \mathbb{P}(y_s)\mathbb{P}(y)$ holds.

    \item $y \perp y_s \, | \, \mathbf{z}$ means given the representation $\mathbf{z}$, the true label $y$ and the spurious label $y_s$ are conditionally independent. %Thus,  $\mathbb{P}(y,y_s|\mathbf{z})=\mathbb{P}(y|\mathbf{z})\mathbb{P}(y_s|\mathbf{z})$ holds.
\end{enumerate}

\end{property}

Naturally, the proxy $\hat y_{s,\mathbf{w}}$ for the true spurious attribute label $y_s$ should also satisfy the aforementioned causal properties.

Then, we establish the following two lemmas:

\begin{lemma}
\label{lemma:theorem-1}
    % \textbf{(Restatement of Lemma \ref{spurious equal})}
    Given representations $\mathbf{z}^{tr}$ and $\mathbf{z}^{c}$, maximizing the spurious term is equivalent to maximizing the following expression:
  \begin{equation}
\begin{split}
     \max_{\mathbf{w}} \mathrm{KL}(\mathbb{P}(y^{tr}|y^{tr}_{s,\mathbf{w}})|| \mathbb{P}(y^{c}|y^{c}_{s,\mathbf{w}})) 
   \Leftrightarrow \max_{\mathbf{w}} \mathrm{KL}(\mathbb{P}(y^{tr},\mathbf{z}^{tr}|y^{tr}_{s,\mathbf{w}})|| \mathbb{P}(y^{c},\mathbf{z}^{c}|y^{c}_{s,\mathbf{w}}))
\end{split}
\end{equation}
\end{lemma}
\begin{proof}
As stated in \cref{sec:GIC: Group Inference via data Comparison}, different variable names are used to emphasize various data sources. In the proof section, we consider to a more formal expression for clarity. That is, we consider different distribution notations with the same variable names and also ignore the subscript $\mathbf{w}$.
\begin{equation}
\begin{split}
    &\mathbb{P}_{tr}(y|\hat y_{s}):=\mathbb P(y^{\mathrm{tr}}|\hat y^{tr}_{s})\\ &\mathbb{P}_{\mathrm{c}}(y|\hat y_{s}):=\mathbb P(y^{c}|
\hat y^{c}_{s}).
\end{split}
\end{equation}
According to the properties of KL divergence \citep{cover1999elements}, we have the following:
\begin{equation}
\begin{split}
    & \mathrm{KL}(\mathbb{P}_{tr}(y,\mathbf{z}|\hat y_{s})|| \mathbb{P}_c(y,\mathbf{z}|\hat y_{s}))
     -\mathrm{KL}(\mathbb{P}_{tr}(\mathbf{z}|y,\hat y_{s})|| \mathbb{P}_c(\mathbf{z}|y,\hat y_{s}))\\
     &=\mathbb{E}_{(y,\mathbf{z},\hat y_{s})}\Big[\log \frac{\mathbb{P}_{tr}(y,\mathbf{z}|\hat y_{s})}{\mathbb{P}_c(y,\mathbf{z}|\hat y_{s})}\Big]-\mathbb{E}_{(y,\mathbf{z},\hat y_{s})}\Big[\log \frac{\mathbb{P}_{tr}(\mathbf{z}|y,\hat y_{s})}{\mathbb{P}_c(\mathbf{z}|y,\hat y_{s})}\Big]\\
     &=\mathbb{E}_{(y,\mathbf{z},\hat y_{s})}\Big[\log \frac{\mathbb{P}_{tr}(y|\hat y_{s})}{\mathbb{P}_c(y|\hat y_{s})}\Big]\\
     &=\mathbb{E}_{(y,\hat y_{s})}\Big[\log \frac{\mathbb{P}_{tr}(y|\hat y_{s})}{\mathbb{P}_c(y|\hat y_{s})}\Big]\\
     &=\mathrm{KL}(\mathbb{P}_{tr}(y|\hat y_{s})|| \mathbb{P}_c(y|\hat y_{s})).
\end{split}
\end{equation}
% \begin{equation}
% \begin{split}
%     & \mathrm{KL}(\mathbb{P}(y^{tr},\mathbf{z}^{tr}|y^{tr}_{s,\mathbf{w}})|| \mathbb{P}(y^{c},\mathbf{z}^{c}|y^{c}_{s,\mathbf{w}}))
%      -\mathrm{KL}(\mathbb{P}(\mathbf{z}^{tr}|y^{tr},\hat y^{tr}_{s})|| \mathbb{P}(\mathbf{z}^{c}|y^{c},y^{c}_{s,\mathbf{w}}))\\
%      &=\mathbb{E}_{(y^{tr},\mathbf{z}^{tr},\hat y^{tr}_{s})}[\log \frac{\mathbb{P}(y^{tr},\mathbf{z}^{tr}|y^{tr}_{s,\mathbf{w}})}{\mathbb{P}(y^{c},\mathbf{z}^{c}|y^{c}_{s,\mathbf{w}})}]-\mathbb{E}_{(y^{tr},\mathbf{z}^{tr},\hat y^{tr}_{s})}[\log \frac{\mathbb{P}(\mathbf{z}^{tr}|y^{tr},\hat y^{tr}_{s})}{\mathbb{P}(\mathbf{z}^{c}|y^{c},y^{c}_{s,\mathbf{w}})}]\\
%      &=\mathbb{E}_{(y^{tr},\mathbf{z}^{tr},\hat y^{tr}_{s})}[\log \frac{\mathbb{P}(y^{tr}|y^{tr}_{s,\mathbf{w}})}{\mathbb{P}(y^{c}|y^{c}_{s,\mathbf{w}})}]\\
%      &=\mathbb{E}_{(y^{tr},\hat y^{tr}_{s})}[\log \frac{\mathbb{P}(y^{tr}|y^{tr}_{s,\mathbf{w}})}{\mathbb{P}(y^{c}|y^{c}_{s,\mathbf{w}})}]\\
%      &=\mathrm{KL}(\mathbb{P}(y^{tr}|y^{tr}_{s,\mathbf{w}})|| \mathbb{P}(y^{c}|y^{c}_{s,\mathbf{w}})) 
% \end{split}
% \end{equation}

Since term $\mathrm{KL}(\mathbb{P}_{tr}(\mathbf{z}|y,\hat y_{s})|| \mathbb{P}_c(\mathbf{z}|y,\hat y_{s}))$ captures the similarity between the same group in the training and validation sets, i.e., $\mathbb{P}_{tr}(\mathbf{z}|y,\hat y_{s}) = \mathbb{P}_c(\mathbf{z}|y,y_{s})$, we have
\begin{equation}
\begin{split}
\mathrm{KL}(\mathbb{P}_{tr}(\mathbf{z}|y,\hat y_{s})|| \mathbb{P}_c(\mathbf{z}|y,\hat y_{s}))
=\sum_{g \in \mathcal{G}} \mathbb{P}_{tr}((y,\hat y_{s})=g) \sum_z \mathbb{P}_{tr}(\mathbf{z}=z|(y,\hat y_{s})=g) \log \frac{\mathbb{P}_{tr}(\mathbf{z}=z|(y,\hat y_{s})=g)}{\mathbb{P}_c(\mathbf{z}=z|(y,y_{s})=g)}
=0.
% &=\sum_{k}r^{tr}_k \sum P_k \log \frac{\mathbb{P}_k}{\mathbb{P}_k}=0
\end{split}
\end{equation}
Therefore, maximizing the spurious term $\mathrm{KL}(\mathbb{P}_{tr}(y|\hat y_{s})|| \mathbb{P}_{c}(y|\hat y_{s}))$  is tantamount to maximizing $\mathrm{KL}(\mathbb{P}_{tr}(y,\mathbf{z}|\hat y_{s})|| \mathbb{P}_c(y,\mathbf{z}|\hat y_{s}))$.
\end{proof}

% \subsection{Proof of Lemma \ref{lower bound of spurious equal}}

\begin{lemma}
\label{lemma:theorem-2}
    % \textbf{(Restatement of Lemma \ref{lower bound of spurious equal})}
    Given representations $\mathbf{z}^{tr}$ and $\mathbf{z}^{c}$, the equivalence of the spurious term is lower bounded by the following expression:
  \begin{equation}
\begin{split}
     \mathrm{KL}(\mathbb{P}(y^{tr},\mathbf{z}^{tr}|y^{tr}_{s,\mathbf{w}})|| \mathbb{P}(y^{c},\mathbf{z}^{c}|y^{c}_{s,\mathbf{w}})) \geq \mathrm{KL}(\mathbb{P}(\mathbf{z}^{tr}|y^{tr}_{s,\mathbf{w}})|| \mathbb{P}(\mathbf{z}^{c}|y^{c}_{s,\mathbf{w}}))
\end{split}
\end{equation}
% The tighter the bound is as $\mathrm{KL}(\mathbb{P}(y^{tr}|\mathbf{z}^{tr})|| \mathbb{P}(y^{c}|\mathbf{z}^{c}))$ becomes smaller.
\end{lemma}
\begin{proof}
We express our variables in the same manner as in \cref{lemma:theorem-1} and omit the subscript $\mathbf{w}$.
\begin{equation}
\label{app:equa-theorem-2}
\begin{split}
     &\mathrm{KL}(\mathbb{P}_{tr}(y,\mathbf{z}|\hat y_{s})|| \mathbb{P}_{c}(y,\mathbf{z}|\hat y_{s}))\\
     &=\mathbb{E}_{(y,\mathbf{z},\hat y_{s})}\Big[ \log \frac{\mathbb{P}_{tr}(y,\mathbf{z}|\hat y_{s})}{\mathbb{P}_c(y,\mathbf{z}|\hat y_{s})}\Big] \\
     &=\mathbb{E}_{(y,\mathbf{z},\hat y_{s})}\Big[ \log \frac{\mathbb{P}{tr}(y,\mathbf{z},\hat y_{s})}{\mathbb{P}_{c}(y,\mathbf{z},y_{s})} \Big] - \mathbb{E}_{\hat y_{s}} \Big[ \log \frac{\mathbb{P}_{tr}(\hat y_{s})}{\mathbb{P}_c(\hat y_{s})} \Big ].
\end{split}
\end{equation}

By the Property \ref{pro:fork causal}, we derive:
\begin{equation}
\begin{split}
    \mathbb{P}(y,\mathbf{z},\hat y_{s})=\mathbb{P}(y|\mathbf{z})\mathbb{P}(\hat y_{s}|\mathbf{z})\mathbb{P}(\mathbf{z})=\mathbb{P}(y|\mathbf{z})\mathbb{P}(\hat y_{s},\mathbf{z}).
\end{split}
\end{equation}
Hence, Equation \eqref{app:equa-theorem-2} can be simplified as follows
\begin{equation}
\label{app:equa-theorem-3}
\begin{split}
     &\mathrm{KL}(\mathbb{P}_{tr}(y,\mathbf{z}|\hat y_{s})|| \mathbb{P}_{c}(y,\mathbf{z}|\hat y_{s}))\\
     &= \mathbb{E}_{(y,\mathbf{z},\hat y_{s})}\Big[ \log \frac{\mathbb{P}_{tr}(y|\mathbf{z})\mathbb{P}_{tr}(\mathbf{z},\hat y_{s})}{\mathbb{P}_{c}(y|\mathbf{z})\mathbb{P}_{c}(\mathbf{z},\hat y_{s})}\Big]-\mathrm{KL}(\mathbb{P}_{tr}(\hat y_{s})||\mathbb{P}_c(\hat y_{s}))\\
     &=\mathbb{E}_{(y,\mathbf{z},\hat y_{s})}\Big[ \log \frac{\mathbb{P}_{tr}(y|\mathbf{z})}{\mathbb{P}_c(y|\mathbf{z})}\Big]+\mathbb{E}_{(y,\mathbf{z},\hat y_{s})}\Big[ \log \frac{\mathbb{P}_{tr}(\mathbf{z},\hat y_{s})}{\mathbb{P}_c(\mathbf{z},\hat y_{s})}\Big]-\mathrm{KL}(\mathbb{P}_{tr}(y_{s})||\mathbb{P}_c(\hat y_{s}))\\
     &=\mathrm{KL}(\mathbb{P}_{tr}(y|\mathbf{z})||\mathbb{P}_c(y|\mathbf{z}))+\mathrm{KL}(\mathbb{P}_{tr}(\mathbf{z}|\hat y_{s})||\mathbb{P}_c(\mathbf{z}|\hat y_{s}))\\
     % &=\mathrm{KL}(\mathbb{P}(y^{tr}|\mathbf{z}^{tr})||\mathbb{P}(y^{c}|\mathbf{z}^{c}))
     &\geq \mathrm{KL}(\mathbb{P}_{tr}\Big(\mathbf{z}|\hat y_{s})||\mathbb{P}_c(\mathbf{z}|\hat y_{s})\Big).
\end{split}
\end{equation}
By Equation \ref{app:equa-theorem-3}, we prove the lower bound of $\mathrm{KL}(\mathbb{P}_{tr}(y,\mathbf{z}|y_{s})|| \mathbb{P}_c(y,\mathbf{z}|y_{s}))$ is $\mathrm{KL}(\mathbb{P}_{tr}(\mathbf{z}|y_{s})||\mathbb{P}_c(\mathbf{z}|y_{s,\mathbf{w}}))$. The proof is complete.
\end{proof}

% and the tighter the bound is as $\mathrm{KL}(\mathbb{P}(y^{tr}|\mathbf{z}^{tr})|| \mathbb{P}(y^{c}|\mathbf{z}^{c}))$ becomes smaller.

\begin{theorem}
% \label{Lower Bound of Spurious Term without Accessible Y}
    \textbf{(Restatement of Theorem \ref{Lower Bound of Spurious Term without Accessible Y})}
    Given representations $\mathbf{z}^{tr}$ and $\mathbf{z}^{c}$, the spurious term is lower bounded by the following expression as:
     \begin{equation}
    \mathrm{KL}(\mathbb{P}(y^{tr}|\hat y^{tr}_{s,\mathbf{w}})|| \mathbb{P}(y^{c}|\hat y^{c}_{s,\mathbf{w}})) \geq \mathrm{KL}(\mathbb{P}(\mathbf{z}^{tr}|\hat y^{tr}_{s,\mathbf{w}})|| \mathbb{P}(\mathbf{z}^{c}|\hat y^{c}_{s,\mathbf{w}}))
\end{equation}
\end{theorem}

\begin{proof}
    Based on \cref{lemma:theorem-1} and \cref{lemma:theorem-2}, we can directly deduce:
    \begin{equation}
    % \label{eq:GIC without y^val}
    \mathrm{KL}(\mathbb{P}(y^{tr}|\hat y^{tr}_{s,\mathbf{w}})|| \mathbb{P}(y^{c}|\hat y^{c}_{s,\mathbf{w}})) \geq \mathrm{KL}(\mathbb{P}(\mathbf{z}^{tr}|\hat y^{tr}_{s,\mathbf{w}})|| \mathbb{P}(\mathbf{z}^{c}|\hat y^{c}_{s,\mathbf{w}})).
\end{equation}
\end{proof}

\section{Experiments}
\label{app:Experiments}
\subsection{Synthetic Toy Data}
\label{app:Synthetic Toy Data}

Synthetic data consists of three sets: training ($\mathcal{D}^{tr}$), validation ($\mathcal{D}^{c}$), and test ($\mathcal{D}^{ts}$). These sets are created by blending four two-dimensional Gaussian distributions (representing four groups) with distinct means, equal variances, and zero correlation coefficients.  Let's consider a 2D synthetic dataset with the following distribution:

\begin{example}
\textbf{(Synthetic 2D data)} Let $\mathbf{x} = (\mathbf{x}_1,\mathbf{x}_2) \in \mathbb{R}^2$ represent 2-dimensional features, with the spurious feature $\mathbf{x}_1$ and the invariant feature $\mathbf{x}_2$, and $y \in \mathbb{R}^1$ denoting labels. The synthetic data comprises four groups, namely $G_1$, $G_2$, $G_3$, and $G_4$. The distributions and sample sizes in the training, validation, and test sets for each group are as follows:
\begin{equation}
% \text{Training Data:}
\begin{cases}
    G_1: (\mathbf{x}_1, \mathbf{x}_2) \sim \mathcal{N}\left(\begin{bmatrix}4 \\ 5\end{bmatrix}, \begin{bmatrix}1 & 0 \\ 0 & 1\end{bmatrix}\right) ; y=0 ; (N^{tr},N^{val},N^{ts})=(3900, 854,3000)\\%X_{1} \sim N(4,1), \;X_2\sim N(5,1), 
     G_2: (\mathbf{x}_1, \mathbf{x}_2) \sim \mathcal{N}\left(\begin{bmatrix}4 \\ 8\end{bmatrix}, \begin{bmatrix}1 & 0 \\ 0 & 1\end{bmatrix}\right); y=1 ; (N^{tr},N^{val},N^{ts})=(100,287,3000) \\  %X_{1} \sim N(4,1), \;X_2\sim N(8,1), 
     G_3: (\mathbf{x}_1, \mathbf{x}_2) \sim \mathcal{N}\left(\begin{bmatrix}8 \\ 8\end{bmatrix}, \begin{bmatrix}1 & 0 \\ 0 & 1\end{bmatrix}\right);y=1 ; (N^{tr},N^{val},N^{ts})=(3900, 18,3000)  \\%   X_{1} \sim N(8,1), \;X_2\sim N(8,1), 
    G_4: (\mathbf{x}_1, \mathbf{x}_2) \sim \mathcal{N}\left(\begin{bmatrix}8 \\ 5\end{bmatrix}, \begin{bmatrix}1 & 0 \\ 0 & 1\end{bmatrix}\right);y=0 ; (N^{tr},N^{c},N^{ts})=(100,828,3000)  \\%X_{1} \sim N(8,1), \;X_2\sim N(5,1), 
\end{cases}
\end{equation} 
  \end{example}
  
The varying sample sizes in groups $G_1$, $G_2$, $G_3$, and $G_4$ indicate different group distributions across the training, validation, and test sets. 

% \cref{fig:Synthetic data} visualizes the synthetic data and annotates the centers of the four groups. In the toy example, we directly use the labeled validation set as the comparison data.

% \begin{figure}[H]
%   \centering
%   \begin{minipage}{0.3\textwidth}
%     \centering
%     \includegraphics[width=1.0\textwidth, height=0.17\textheight]{ICLR/png/toy2_train-compressed.pdf}
%     \subcaption{Synthetic training data.}
%     % \label{fig:toy example}
%   \end{minipage}%
%   \begin{minipage}{0.3\textwidth}
%     \centering
%     \includegraphics[width=\textwidth]{ICLR/png/toy2_val-compressed.pdf}
%     \subcaption{Synthetic Validation Data.}
%     % \label{fig:sacc}
%   \end{minipage}
%    \begin{minipage}{0.3\textwidth}
%     \centering
%     \includegraphics[width=\textwidth]{ICLR/png/toy2_test-compressed.pdf}
%     \subcaption{Synthetic Test Data.}
%     % \label{fig:sacc}
%   \end{minipage}
% \caption{Visualization of synthetic data with four groups $\{G_1,G_2,G_3,G_4\}$. The group distributions of synthetic training data, validation data (comparison data), and test data are different.}
% \label{fig:Synthetic data}
% \end{figure}

\subsection{Real Data}
\label{app:high-dimensional data}

\subsubsection{Dataset Details}
\label{app:Dataset Details}
\textbf{CMNIST}\citep{arjovsky2019invariant} is a noisy digit recognition task. The binary feature (green and red), referred to as color, serves as a spurious feature, while the binary feature (digit contours) acts as the invariant feature. The CMNIST dataset involves two classes, where class 0 corresponds to the original digits (0,1,2,3,4), and class 1 represents digits (5,6,7,8,9). Following the approach recommended in \cite{yao2022improving}, we construct a training set with a sample size of 30,000. In class 0, the ratio of red to green samples is set at 8:2, while in class 1, it is set at 2:8. For the validation set consisting of 10,000 samples, the proportion of green to red samples is equal at 1:1 for all classes. The test set, containing 20,000 samples, features a proportion of green to red samples at 1:9 in class 0 and 9:1 in class 1. Additionally, label flipping is applied with a probability of 0.25. In our experiments, we utilized the validation set as the comparison data and employed an unlabeled validation set to simulate scenarios where comparison data is unavailable. The four groups of CMNIST is $(g_1,g_2,g_3,g_4)= (\{0,green\},\{0,red\},\{1,green\},\{1,red\})$ and the group distribution of the training data is $(g_1,g_2,g_3,g_4)=(0.1,0.4,0.4,0.1)$, while the group distribution of the comparison data is $(g_1,g_2,g_3,g_4)=(0.26,0.25,0.25,0.24)$.

\textbf{Waterbirds} aims to classify bird images as either waterbirds or landbirds, with each bird image falsely associated with either a water or land background. Waterbirds is a synthetic dataset where each image is generated by combining bird images sampled from the CUB dataset \citep{wah2011caltech} with backgrounds selected from the Places dataset \citep{zhou2017places}. We directly load the Waterbirds dataset using the Wilds library in PyTorch \citep{koh2021wilds}. The dataset consists of a total of 4,795 training samples, with only 56 samples labeled as waterbirds on land and 184 samples labeled as landbirds on water. The remaining training data includes 3,498 samples from landbirds on land and 1,057 samples from waterbirds on water. We still directly use the validation set as comparison data. The waterbirds dataset can be divided into four groups, namely $(g_1,g_2,g_3,g_4)= (\{landbird,land\},\{landbird,water\},\{waterbird,land\},\{waterbird,water\})$. The class distribution of the training data is $(g_1,g_2,g_3,g_4)=(0.73,0.04,0.10,0.13)$, while the class distribution of the comparison data is $(g_1,g_2,g_3,g_4)=(0.56,0.22,0.16,0.16)$.

\textbf{CelebA} \citep{liu2015deep,sagawa2019distributionally} is a hair-color prediction task, similar to the study conducted by \citep{yao2022improving}, and follows the data preprocessing procedure outlined in \citep{sagawa2019distributionally}. Given facial images of celebrities as input, the task is to identify their hair color as either blond or non-blond. This labeling is spuriously correlated with gender, which can be either male or female. In the training set, there are 71,629 instances (44\%) of females with non-blond hair, 66,874 instances (41\%) of non-blond males, 22,880 instances (14\%) of blond females, and 1,387 instances (1\%) of blond males.  In the validation set, there are 8535 instances (43\%) of females with non-blond hair, 8276 instances (42\%) of non-blond males, 2874 instances (14\%) of blond females, and 182 instances (1\%) of blond males. The validation set is still regarded as the comparison data to infer group labels. We found that the group distribution of the training data and the comparison data in the CelebA dataset are very similar, which poses a challenge for GIC inference.
 
\textbf{CivilComments-WILDS}\citep{koh2021wilds} used in our experiments is derived from the Jigsaw dataset \cite{borkan2019nuanced}. The task of CivilComments is to determine whether a given online comment is toxic. The attribute of interest, denoted as $a$, is an 8-dimensional binary vector, where each entry is a binary indicator representing whether the online comment mentions one of the following 8 demographic identities: $a \in$\{male, female, LGBTQ, Christian, Muslim, other religion, Black, White\}. Comments that mention these identits tend to be more offensive and toxic. Similar to JTT \cite{liu2021just}, during the training of GIC, we binarize the spurious attribute labels, specifically using the dataset with ``identity any" as the only identity attribute \citep{koh2021wilds}. And during the evaluation phase for worst-group accuracy, we evaluate over these 16 groups $\{(y,a_j)\}^8_{j=1}$ to calculate the wrost accuracy, where the true label $y$ represents toxic or non-toxic comments. After binarizing the spurious attributes of the original data, there are 4 groups $(g_1,g_2,g_3,g_4)=\{$non-toxic no identities, non-toxic has identities, toxic no identities, toxic has identities$\}$. However, the CivilComments dataset actually contains 16 subgroups , which can more finely characterize the group distribution differences between the training data and comparison data. Among these, the ratios of majority to minority groups in the training and validation datasets are $16.9:1$ and $31.6:1$, respectively. This can more directly reflect the distribution differences between the training and validation data compared to the binarization groups.

\subsubsection{Baseline Details}
\label{Downstream invariant learning methods}
In this section, we main focus on baselines in this paper, categorizing them into two types based on the need for group labels: mitigating spurious correlations with group labels and mitigating spurious correlations without group labels.

\textbf{Spurious correlations mitigation with group labels.} GroupDRO \cite{sagawa2019distributionally} is a well-established method that enhances worst-group accuracy using group labels. Unlike standard ERM, which minimizes the average loss across all groups, GroupDRO partitions the data based on prior knowledge of groups and minimizes the worst-case loss over groups in the training data, thereby improving worst-group accuracy. DFR \cite{kirichenko2022last} balances the training set by leveraging the Subsample strategy. This involves retaining all data from the smallest group and Subsample data from the other groups to equalize group sizes, followed by retraining the ERM model using this balanced dataset. LISA \cite{yao2022improving} addresses spurious correlation between the spurious attribute and true label through Mixup. It employs two mixup strategies: Intra-label, which interpolates samples with the same label but different spurious attributes, and intra-domain, which interpolates samples with the same spurious attributes but different true labels. Through these mixup techniques, LISA achieves a significant improvement in the worst-group accuracy. It is important to emphasize that LISA utilizes different Mixup techniques for various experimental datasets. Specifically, for CMNIST and Waterbirds, the classic mixup (i.e., linear interpolation between two images) is considered. In the case of CelebA, CutMix, which involves replacing removed regions with a patch from another image, is employed. Meanwhile, for CivilComments, LISA resorts to using Manifold Mix. For all four datasets in the experiments, GIC adopts the same Mixup strategies as LISA.

\textbf{Spurious correlations mitigation without group labels.} JTT \cite{liu2021just} initially treats misclassified points by a trained ERM as errors, which are then upsampled for improving model robustness. Similarly, CnC \cite{zhang2022correct} uses the outputs of an ERM model to identify samples with the same class but different spurious features, and trains a robust model with contrastive learning. EIIL \cite{creager2021environment} learns spurious attribute (group) labels using the EI method, which maximizes the loss of GroupDRO, and utilizes GroupDRO to learn invariant features. However, its performance is unstable and relatively poor in scenarios where learning spurious correlations is challenging \cite{lin2022zin}. SSA \cite{nam2022spread} leverages semi-supervised learning with available group labels for some samples to train a spurious attribute predictor. ZIN \cite{lin2022zin} improves group inference with auxiliary information such as timestamps for time-series data or meta-annotations for images, while DISC \cite{wu2023discover} constructs a concept bank with potential spurious attributes. However, acquiring such auxiliary information (human priors or partial group labels) may also pose difficulties, limiting the general applicability of these methods.

\textbf{Downstream invariant learning methods.} We considere four downstream invariant learning algorithms. Mixup \cite{yao2022improving} selectively mixes samples with the same label but different spurious features or samples with different labels but the same spurious feature, to achieve invariance. Subsample \cite{kirichenko2022last} constructs balanced data by keeping the smallest sample size group and subsampling other groups to the same quantity. Upsample \cite{liu2021just} oversamples samples in the misclassified sets using ERM, emphasizing difficult-to-predict groups, as well as the classic GrouDRO \cite{sagawa2019distributionally} algorithm.

\subsubsection{Training Details}
\label{app:Training Details}
This section describes the experimental details, including hyperparameters and model architectures. We follow the stages mentioned in Algorithm \ref{alg:GIC} to describe the training details.

\textbf{Stage 1: Extracting feature representations.} First, we introduce the feature extractors $\Phi(\cdot)$ used for each dataset. We follow previous works \cite{kirichenko2022last, liu2021just, zhang2022correct, yao2022improving} and consider architectures including LeNet-5 CNN (CMNIST), ResNet-50 (Waterbirds and CelebA), and BERT (CivilComments) as our feature extractors.  We detail the feature extractors architecture and hyperparameters for each dataset below:
\begin{enumerate}
    \item CMNIST: We use the LeNet-5 CNN architecture in the pytorch image classification tutorial. We train with SGD, epochs $E = 5$, learning rate $1e-3$, batch size 32, default weight decay $1e-4$, and momentum 0.9.
    \item Waterbirds: We use the \texttt{torchvision} implementation of ResNet-50 with pretrained weights from ImageNet. Also as previous works , we train with SGD, default epochs $E = 100$, learning rate $1e-3$, weight decay $1e-3$, batch size 32, and momentum 0.9.
    \item CelebA: We directly use the \texttt{torchvision} implementation of ResNet-50 with pretrained weights from ImageNet without any fine-tuning as the feature extractor.
    \item CivilComments:  We use the HuggingFace (\texttt{pytorch-transformers}) implementation of BERT with pre-trained weights and number of tokens capped at 220 without any fine-tuning. 
\end{enumerate}

The CelebA and CivilComments datasets directly use pretrained models as feature extractors without any fine-tuning, and CMNSIT trains the extractor from scratch. It is worth noting that the Waterbirds dataset utilizes pretrained weights from ImageNet and then fine-tunes them on the Waterbirds dataset. We adopt this strategy because previous work \cite{kirichenko2022last} has observed that initializing the feature extractor with ImageNet trained weights and fine-tuning on the dataset can significantly improve the performance of feature extraction on the Waterbirds dataset. However, such improvements are not significant on datasets with larger sample sizes, such as CelebA.

\textbf{Stage 2:  Inferring group labels.} 
In this section, we describe the model architectures and training hyperparameters used for training the spurious attribute classifier $f_{\mathrm{GIC}}$. We also explain how we selected crucial hyperparameters, such as the weighting parameter $\gamma$ and the training epochs $K$.

As the input to $f_\mathrm{GIC}$ are simple linear embeddings (similar to the input of the last layer of a neural network) after feature extraction, we only use a simple single-layer neural network as the model structure of $f_\mathrm{GIC}$ on all datasets. Specifically, we mapped the input to a 2-dimensional space and then used the Sigmoid function to convert it into a probability as the prediction of spurious attributes. The hyperparameters used for each dataset in Stage 2 are as follows:

\begin{enumerate}
    \item CMNIST: We use training epochs $K=20$, a weight of $\gamma = 10$, a learning rate of $lr=1e-5$, and a momentum$ = 0.9$ .
    \item Waterbirds: We use training epochs $K=10$, a weight of $\gamma = 10$, a learning rate of $lr=1e-4$, and a momentum$ = 0.9$.
    \item CelebA: We use training epochs $K=15$, a weight of $\gamma = 10$, a learning rate of $lr=1e-5$, and a momentum$ = 0.9$.
     \item CivilComments: We use training epochs $K=5$, a weight of $\gamma = 5$, a learning rate of $lr=1e-5$, and a momentum$ = 0.9$.
\end{enumerate}

In the above parameters, the most important ones are the weight parameter $\gamma$ and the number of training epochs $K$. The weight parameter adjusts the balance of the correlated and spurious terms, and a large $\gamma$t can cause the GIC to overly focus on the spurious term, resulting in an increase in the correlated CE loss. On the other hand, a small $\gamma$ may not facilitate the learning of spurious features. Additionally, when fixing the weight parameter, we observe the spurious loss of may decrease and the correlated term loss increases with too large training epochs $K$. Therefore, it is crucial to find a reasonable choices both for the weight parameter $\gamma$ and training epochs $K$. We propose a grid search method for selecting $\gamma$ and $K$ by observing the trends of KL Loss and CE Loss. We start with $\gamma = 10$ and $K=20$, and compute the corresponding KL loss (spurious term loss) and CE loss (relevant term loss). Our goal is to maximize the KL loss and minimize the CE loss. If we observe a decrease in the KL loss and an increase in the CE loss, we first decrease $K \in \{20,15,14,\cdots,2,1\}$. Once $K$ cannot be further reduced, we then decrease $\gamma \in \{10,5,4,3,2,1\}$ and adjust $K$ accordingly. We repeat this process until we no longer observe a decrease in the KL loss and an increase in the CE loss, and the corresponding values of $\gamma$ and $K$ at this point are the selected parameters. Taking the Waterbirds dataset as an example, \cref{fig:loss} visualizes the changes in KL loss and CE loss when comparing data without labels. We observe that when $\gamma = 10$ and $K = 20$, the CE loss does not decrease significantly and even starts to increase around epoch $= 10$. This prompts us to reduce the number of training iterations $K$ to 10.

\begin{figure}[H]
  \centering
  \includegraphics[width=.5\textwidth, height=0.13\textheight]{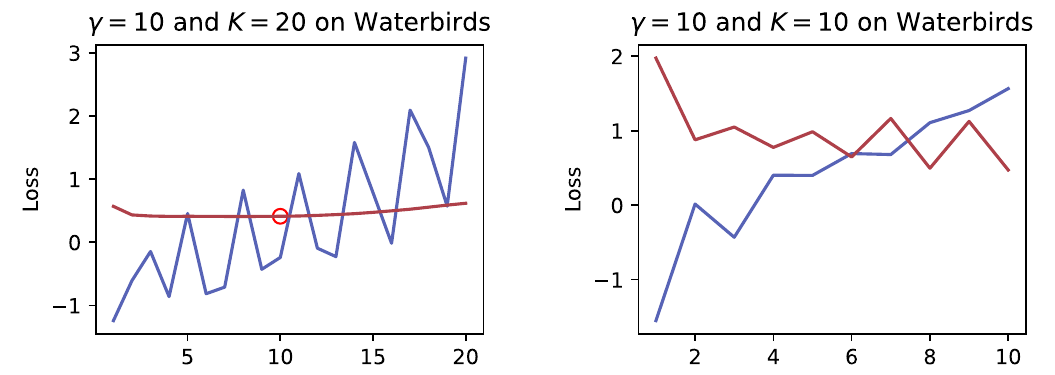}
% width=1.0\textwidth, 
  \caption{The selection of important parameters $\gamma$ and $K$ in DIG on Waterbirds. The red circle indicates the inflection point where the CE Loss starts to increase. We observe that excessively large values of $K$ and $\gamma$ would cause the CE Loss to increase. Therefore, by observing the changing trends of the CE Loss and KL Loss, we determined the optimal parameters using grid search.}
  \label{fig:loss}
  % \vspace{-10pt}
\end{figure}

\textbf{Stage 3: Invariant learning.} After obtaining the inferred group labels using GIC, the final stage is the invariant learning stage. In the third stage, we train the robust model using the inferred group labels from the second stage for invariant learning. In our experiments, we consider four invariant learning algorithms: Mixup, GroupDRO, Subsample, and Upsample. During this stage, we use the parameters and model architecture corresponding to each specific invariant learning algorithm. For Mixup, we directly use the parameters from the original LISA paper \cite{yao2022improving}. We only need to replace the spurious class labels of the original training set with the labels predicted by GIC. For GroupDRO and Subsample, we follow the parameters implemented in EIIL \cite{creager2021environment}, CnC \cite{zhang2022correct}, and DFR \cite{kirichenko2022last}. As for Upsample, we refer to the parameters selected by JTT \cite{liu2021just}.

In Phase 3, in line with baselines such as JTT, CnC, EIIL, and SSA, which tune all hyperparameters and also apply early stopping based on the highest worst-group accuracy on the validation set, we similarly use validation data with group labels for model selection. Specifically, GIC applies early stopping based on the highest worst-group accuracy on the validation set, thus establishing the final count of training epochs.

\section{More Results}
\label{more results}

In this section, we provide additional experimental results as a supplement to Section \ref{Experiments}. Firstly, we fully display the average-group accuracy and the worst-group accuracy as outlined in \cref{The scalability of GIC}. Subsequently, we showcase the results of ablation experiments, illustrating the performance enhancement of GIC with the adoption of early stopping. Then, we present further error cases from the Waterbirds and CelebA datasets, underscoring the prevalent manifestation of semantic consistency in GIC.Following that, we discuss the various sources of comparison data and the impact of their sample sizes as mentioned in \cref{The Construction of Comparison Data}, to elucidate the effect of distinct origins and quantities of comparison data on the performance of GIC. Finally, we highlight the performance implications of the distribution discrepancy between the comparison data and training data on GIC, augmenting the insights in \cref{Discussion}.

\subsection{Complete Results of group inference method comparison}
\label{Complete Results of group inference method comparison}

In this section, we present the experimental results that were not fully shown in \cref{The scalability of GIC} due to space limitations. These results include the average-group accuracy, where GIC outperforms baselines not only in the worst-performing group but also achieves higher accuracy in the average group. This further demonstrates that GIC has the ability to improve both the average and worst-group accuracy.

\begin{table}[H]
% \vspace{-0.2cm}
\setlength{\abovecaptionskip}{-0.01cm}
\centering
\caption{Average and worst-group accuracy comparison (\%). }
\vskip 0.1in
{
\begin{tabular}{@{}lcc|cc|cc|cccccc@{}}
\toprule
\multirow{2}{*}{Method} &   \multicolumn{2}{c}{\makecell{Waterbirds}}& \multicolumn{2}{c}{\makecell{CelebA}}& \multicolumn{2}{c}{\makecell{Waterbirds}}& \multicolumn{2}{c}{\makecell{CelebA}}\\
& Avg. & Worst & Avg. & Worst & Avg. & Worst & Avg. & Worst \\
\midrule
\multicolumn{5}{c}{\cellcolor{gray!25}\textit{+GroupDRO}} & \multicolumn{5}{c}{\cellcolor{gray!25}\textit{+Subsample}} \\
\midrule
ERM  &94.6$\pm \scriptstyle{0.0}$&75.6$\pm \scriptstyle{0.4}$&85.9$\pm \scriptstyle{0.1}$&77.2$\pm \scriptstyle{0.1}$&91.3$\pm \scriptstyle{1.8}$&79.4$\pm \scriptstyle{0.3}$ &88.9$\pm \scriptstyle{0.6}$&78.5$\pm \scriptstyle{0.1}$&\\
EI&96.5$\pm \scriptstyle{0.2}$ & 77.2$\pm \scriptstyle{1.0}$&85.7$\pm \scriptstyle{0.1}$& 81.7$\pm \scriptstyle{0.8}$&88.1$\pm \scriptstyle{0.3}$ &81.9$\pm \scriptstyle{1.4}$ &90.9$\pm \scriptstyle{0.1}$&82.8$\pm \scriptstyle{0.5}$\\
$\mathrm{GIC}_{\mathcal{C}_y}$&97.2$\pm \scriptstyle{0.6}$& 80.2$\pm \scriptstyle{0.1}$&92.7$\pm \scriptstyle{1.0}$ & 82.1$\pm \scriptstyle{0.3}$&89.2$\pm \scriptstyle{0.9}$& 83.5$\pm \scriptstyle{0.8}$&91.1$\pm \scriptstyle{0.1}$ & 86.1$\pm \scriptstyle{2.2}$\\
$\mathrm{GIC}_{\mathcal{C}}$&97.6$\pm \scriptstyle{0.1}$& 79.2$\pm \scriptstyle{0.4}$&93.8$\pm \scriptstyle{0.2}$ & 79.7$\pm \scriptstyle{0.6}$ &89.8$\pm \scriptstyle{1.3}$ & 82.1$\pm \scriptstyle{1.1}$ &90.3$\pm \scriptstyle{0.3}$& 83.1$\pm \scriptstyle{0.3}$\\
\midrule
\multicolumn{5}{c}{\cellcolor{gray!25}\textit{+Upsample}}      & \multicolumn{5}{c}{\cellcolor{gray!25}\textit{+Mixup}} \\
\midrule
ERM&89.3$\pm \scriptstyle{0.7}$ & 83.8$\pm \scriptstyle{1.2}$&88.1$\pm \scriptstyle{0.3}$& 81.5$\pm \scriptstyle{1.7}$& 94.0$\pm \scriptstyle{0.4}$ & 82.1$\pm \scriptstyle{0.8}$ &90.5$\pm \scriptstyle{0.3}$&80.6$\pm \scriptstyle{1.7}$\\
EI &88.8$\pm \scriptstyle{0.3}$&81.3$\pm \scriptstyle{0.7}$&95.4$\pm \scriptstyle{0.2}$&84.8$\pm \scriptstyle{0.2}$&90.1$\pm \scriptstyle{0.3}$ & 85.7$\pm \scriptstyle{0.4}$ &90.7$\pm \scriptstyle{0.6}$&84.9$\pm \scriptstyle{3.7}$\\
$\mathrm{GIC}_{\mathcal{C}_y}$ &91.4$\pm \scriptstyle{0.3}$& 84.1$\pm \scriptstyle{0.0}$&88.5$\pm \scriptstyle{0.7}$ & 87.2$\pm \scriptstyle{0.0}$& 89.6$\pm \scriptstyle{1.3}$ & 86.3$\pm \scriptstyle{0.1}$ &91.9$\pm \scriptstyle{0.1}$&89.4$\pm \scriptstyle{0.2}$\\
$\mathrm{GIC}_{\mathcal{C}}$ &90.8$\pm \scriptstyle{0.2}$& 82.1$\pm \scriptstyle{0.7}$&89.7$\pm \scriptstyle{0.0}$&87.8$\pm \scriptstyle{1.1}$&89.3$\pm \scriptstyle{0.8}$&  85.4$\pm \scriptstyle{0.1}$&92.1$\pm \scriptstyle{0.1}$&89.5$\pm \scriptstyle{0.0}$\\
\bottomrule
\end{tabular}}

\label{tab:acc-2}
 % \vspace{-10pt}
\end{table}

\subsection{Ablation Study}
\label{Ablation Study}

In this section, We present ablation experiments results. For each invariant learning algorithms, including Mixup, GroupDRO, Upsample, and Subsample, we followe previous research \cite{liu2021just,yao2022improving,zhang2022correct,creager2021environment},and us the group labels of the validation set and consider the worst-group accuracy to determine early stopping. Tables \ref{tab:acc-abla-1} and \ref{tab:acc-abla-2} show the average and worst-performing group performances obtained by combining GIC with different invariant learning algorithms, both with and without early stopping. We observe that the reference for early stopping does indeed improve performance, especially for the worst-group.

\begin{table}[H]
% \vspace{-0.2cm}
\setlength{\abovecaptionskip}{-0.01cm}
\centering
\caption{Ablation experimental results. The results of combining GIC with Mixup, using early stopping based on validation set labels. Table \ref{tab:acc-abla-1} shows the performance improvement brought by early stopping.}
\vskip 0.1in
% \setlength{\tabcolsep}{1.5pt} %
% {
{
\begin{tabular}{@{}lccc|cc|cc|cccc@{}}
\toprule
\multirow{2}{*}{Method} & \multirow{2}{*}{\makecell{ES}} & \multicolumn{2}{c}{CMNIST} &  \multicolumn{2}{c}{\makecell{Waterbirds}}& \multicolumn{2}{c}{\makecell{CelebA}} & \multicolumn{2}{c}{\makecell{CivilComments}}\\
&  & Avg.& Worst & Avg. & Worst& Avg. & Worst& Avg. & Worst\\
\midrule
$\mathrm{GIC}_{\mathcal{C}_y}$-M& $\times$ & 70.3$\pm \scriptstyle{0.4}$ & 65.1$\pm \scriptstyle{1.1}$ &92.6$\pm \scriptstyle{0.6}$&80.1$\pm \scriptstyle{0.1}$&91.3$\pm \scriptstyle{0.8}$&86.9$\pm \scriptstyle{1.4}$&90.8$\pm \scriptstyle{0.5}$ & 67.6$\pm \scriptstyle{1.1}$\\
$\mathrm{GIC}_{\mathcal{C}}$-M & $\times$ & 67.5$\pm \scriptstyle{1.5}$ & 63.1$\pm \scriptstyle{0.8} $ &89.8$\pm \scriptstyle{0.5}$&80.5$\pm \scriptstyle{0.0}$&92.1$\pm \scriptstyle{0.0}$&85.6$\pm \scriptstyle{1.2}$&90.9$\pm \scriptstyle{0.2}$&66.6$\pm \scriptstyle{0.3}$\\
$\mathrm{GIC}_{\mathcal{C}_y}$-M & $\checkmark$  &{73.2}$\pm \scriptstyle{0.2}$&72.2$\pm \scriptstyle{0.5}$ &89.6$\pm \scriptstyle{1.3}$ & 86.3$\pm \scriptstyle{0.1}$ &91.9$\pm \scriptstyle{0.1}$&89.4$\pm \scriptstyle{0.2}$&90.0$\pm \scriptstyle{0.2}$&72.5$\pm \scriptstyle{0.3}$\\
$\mathrm{GIC}_{\mathcal{C}}$-M & $\checkmark$  & 73.1$\pm \scriptstyle{0.5}$&71.7$\pm \scriptstyle{0.3}$ &89.3$\pm \scriptstyle{0.8}$&  85.4$\pm \scriptstyle{0.1}$&92.1$\pm \scriptstyle{0.1}$&89.5$\pm \scriptstyle{0.0}$&89.7$\pm \scriptstyle{0.0}$&72.3$\pm \scriptstyle{0.2}$\\

\bottomrule
\end{tabular}}
\label{tab:acc-abla-1}
 % \vspace{-10pt}
\end{table}

\begin{table}[H]
% \vspace{-0.2cm}
\setlength{\abovecaptionskip}{-0.01cm}
\centering
\caption{Ablation experimental results. The results of combining GIC with GroupDRO, Subsample, Upsample and Mixup, using early stopping based on validation set labels. Table \ref{tab:acc-abla-2} displays the early stopping results of the experiments in \cref{tab:acc-2}, demonstrating the effectiveness of early stopping.}
\vskip 0.1in
% \setlength{\tabcolsep}{1.5pt} %
% {
{
\begin{tabular}{@{}lccc|cc|cc|cccccc@{}}
\toprule
\multirow{2}{*}{Method} &\multirow{2}{*}{\makecell{ES}} & \multicolumn{2}{c}{\makecell{Waterbirds}}& \multicolumn{2}{c}{\makecell{CelebA}}& \multicolumn{2}{c}{\makecell{Waterbirds}}& \multicolumn{2}{c}{\makecell{CelebA}}\\
&  & Avg. & Worst& Avg. & Worst&Avg. & Worst& Avg. & Worst\\
\midrule
\multicolumn{6}{c}{\cellcolor{gray!25}\textit{+GroupDRO}}      & \multicolumn{4}{c}{\cellcolor{gray!25}\textit{+Subsample}} \\
\midrule
$\mathrm{GIC}_{\mathcal{C}_y}$&$\times$&98.0$\pm \scriptstyle{0.0}$&77.1$\pm \scriptstyle{1.8}$&94.3$\pm \scriptstyle{0.1}$&76.7$\pm \scriptstyle{0.6}$&88.9$\pm \scriptstyle{0.4}$&80.0$\pm \scriptstyle{1.4}$&89.9$\pm \scriptstyle{0.2}$ & 84.1$\pm \scriptstyle{2.9}$\\
$\mathrm{GIC}_{\mathcal{C}}$&$\times$&97.9$\pm \scriptstyle{0.0}$&70.1$\pm \scriptstyle{1.2}$&93.9$\pm \scriptstyle{0.0}$&71.7$\pm \scriptstyle{0.0}$&90.4$\pm \scriptstyle{0.4}$&79.6$\pm \scriptstyle{1.6}$&92.1$\pm \scriptstyle{0.5}$& 77.6$\pm \scriptstyle{3.3}$\\
$\mathrm{GIC}_{\mathcal{C}_y}$&$\checkmark$&97.2$\pm \scriptstyle{0.6}$& 80.2$\pm \scriptstyle{0.1}$&92.7$\pm \scriptstyle{1.0}$ & 82.1$\pm \scriptstyle{0.3}$&89.2$\pm \scriptstyle{0.9}$& 83.5$\pm \scriptstyle{0.8}$&91.1$\pm \scriptstyle{0.1}$ & 86.1$\pm \scriptstyle{2.2}$\\
$\mathrm{GIC}_{\mathcal{C}}$&$\checkmark$&97.6$\pm \scriptstyle{0.1}$& 79.2$\pm \scriptstyle{0.4}$&93.8$\pm \scriptstyle{0.2}$ & 79.7$\pm \scriptstyle{0.6}$ &89.8$\pm \scriptstyle{1.3}$ & 82.1$\pm \scriptstyle{1.1}$ &90.3$\pm \scriptstyle{0.3}$& 83.1$\pm \scriptstyle{0.3}$\\
\midrule
\multicolumn{6}{c}{\cellcolor{gray!25}\textit{+Upsample}}      & \multicolumn{4}{c}{\cellcolor{gray!25}\textit{+Mixup}} \\
\midrule
$\mathrm{GIC}_{\mathcal{C}_y}$&$\times$&85.8$\pm \scriptstyle{1.0}$&77.7$\pm \scriptstyle{0.1}$&86.9$\pm \scriptstyle{0.9}$&84.2$\pm \scriptstyle{1.4}$&92.6$\pm \scriptstyle{0.6}$&80.1$\pm \scriptstyle{0.1}$&91.3$\pm \scriptstyle{0.8}$&86.9$\pm \scriptstyle{1.4}$\\
$\mathrm{GIC}_{\mathcal{C}_y}$&$\times$ &83.3$\pm \scriptstyle{2.9}$&70.9$\pm \scriptstyle{0.6}$&89.9$\pm \scriptstyle{0.1}$&87.5$\pm \scriptstyle{1.4}$&89.8$\pm \scriptstyle{0.5}$&80.5$\pm \scriptstyle{0.0}$&92.1$\pm \scriptstyle{0.0}$&85.6$\pm \scriptstyle{1.2}$\\
$\mathrm{GIC}_{\mathcal{C}_y}$ &$\checkmark$&91.4$\pm \scriptstyle{0.3}$& 84.1$\pm \scriptstyle{0.0}$&88.5$\pm \scriptstyle{0.7}$ & 87.2$\pm \scriptstyle{0.0}$& 89.6$\pm \scriptstyle{1.3}$ & 86.3$\pm \scriptstyle{0.1}$ &91.9$\pm \scriptstyle{0.1}$&89.4$\pm \scriptstyle{0.2}$\\
$\mathrm{GIC}_{\mathcal{C}}$ &$\checkmark$ &90.8$\pm \scriptstyle{0.2}$& 82.1$\pm \scriptstyle{0.7}$&89.7$\pm \scriptstyle{0.0}$&87.8$\pm \scriptstyle{1.1}$&89.3$\pm \scriptstyle{0.8}$&  85.4$\pm \scriptstyle{0.1}$&92.1$\pm \scriptstyle{0.1}$&89.5$\pm \scriptstyle{0.0}$\\
\bottomrule
\end{tabular}}
\label{tab:acc-abla-2}
 % \vspace{-10pt}
\end{table}

\subsection{More Error Cases}
\label{app:Error Case}

In \cref{Error Case}, we illustrate the semantic consistency of GIC using the celebA dataset, where images with similar semantics are considered as belonging to the same spurious feature category. This distinguishes GIC, which relies on semantic features for recognition, from human-based oracle group labels. We observe the same phenomenon on the waterbirds dataset, as shown in \cref{fig:waterbirds}. The misclassifications of GIC can be categorized into two types: (1) Land background with typical water features, such as extensive blue regions, often leads GIC to misclassify land as water. (2) Water background with typical land features, such as abundant tree branches, ponds with lush green vegetation, or large tree reflections, frequently results in GIC classifying them as land. We believe that such consistency precisely reflects GIC's accurate recognition of image semantics. As mentioned in the main text, such accurate recognition is crucial when it comes to invariant learning that requires leveraging different image semantics (e.g., using mixup to disrupt semantic features of spurious features). We present more examples of misclassifications by GIC from the Waterbirds and CelebA datasets in \cref{fig:error-1}, \ref{fig:error-2}, and \ref{fig:error-3} to support the existence of semantic consistency in GIC.

\begin{figure}[H]
  \centering
  \includegraphics[width=.9\textwidth, height=0.15\textheight]{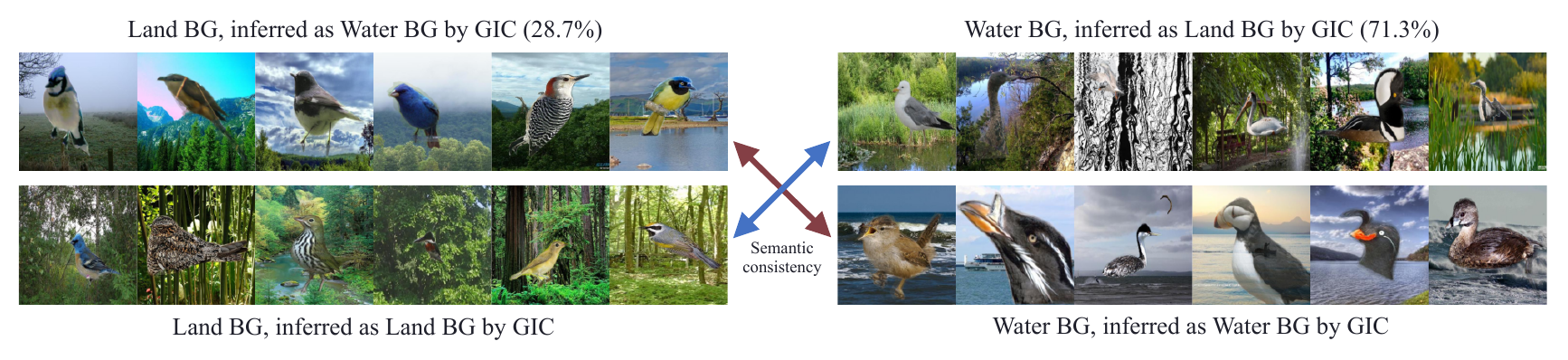}
% width=1.0\textwidth, 
  \caption{Misclassified samples on Waterbirds. The semantic consistency in GIC leads to the misclassification of water with a large amount of land elements (such as twigs and greenery) as land (71.3\%), and land with a large blue area (a typical characteristic of water) as water (28.7\%).}
  \label{fig:waterbirds}
\end{figure}

\begin{figure}[H]
  \centering
  \includegraphics[width=.9\textwidth, height=0.4\textheight]{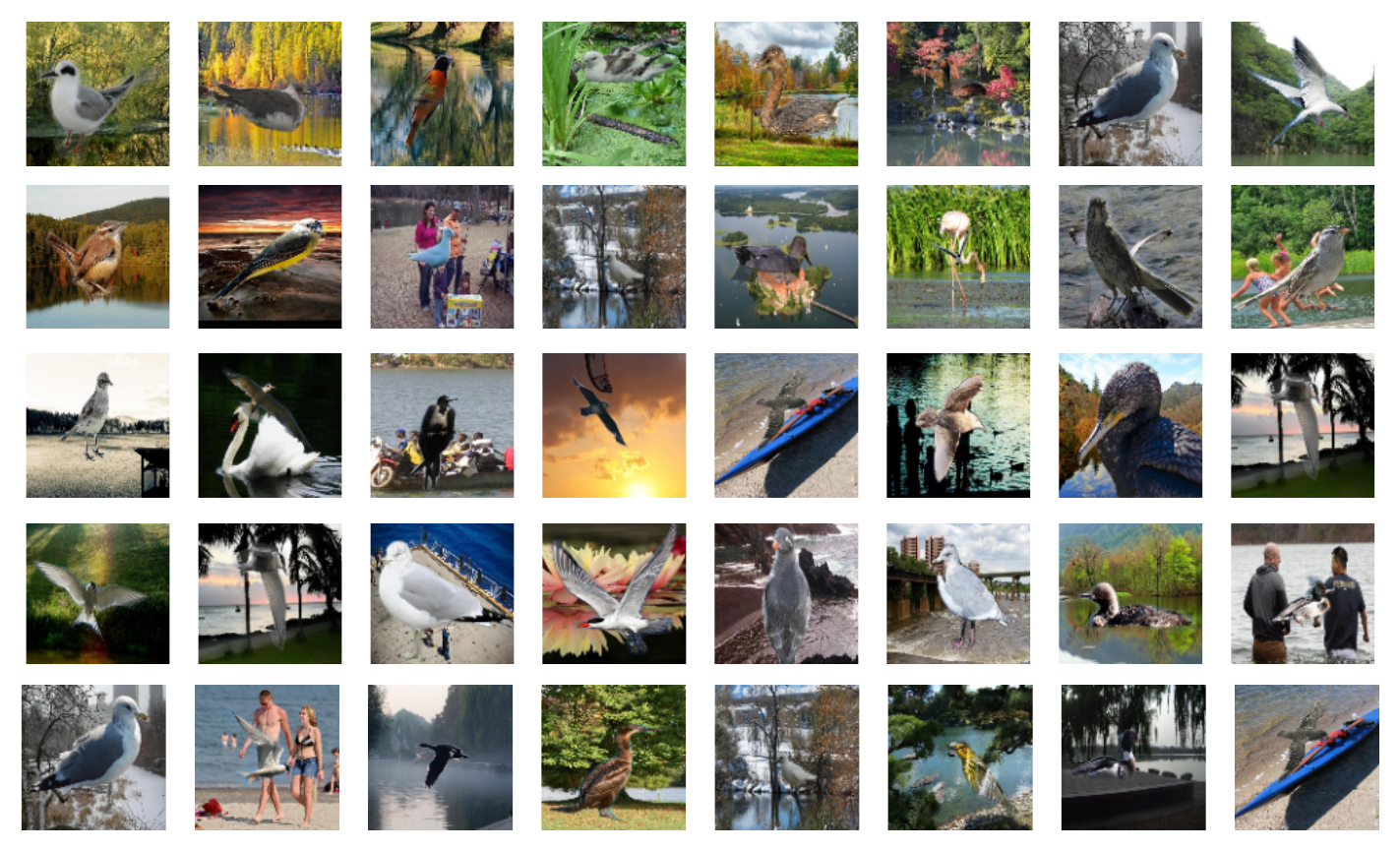}
  \caption{More error cases on Waterbirds. Instances misclassified by GIC as land but are actually water. We find that most of these samples share common characteristics, such as having abundant green vegetation, striped patterns (tree trunks, branches, humans), or locations that combine land and water features (such as beaches, coastal cities). These typical land features are the reasons for GIC's misclassification.}
  \label{fig:error-1}
\end{figure}

\begin{figure}[H]
  \centering
  \includegraphics[width=.9\textwidth, height=0.4\textheight]{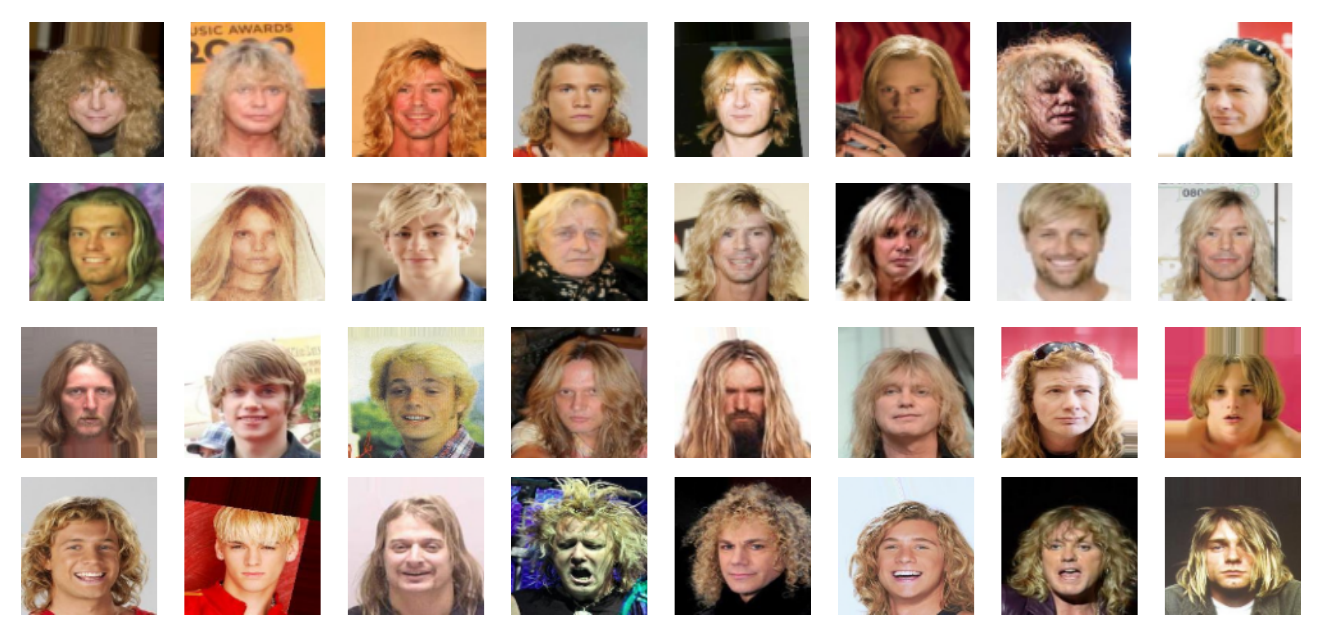}
% width=1.0\textwidth, 
  \caption{More error cases on CelebA. Instances misclassified by GIC as woman but are actually man. We notice that most of them are males with long hair. The presence of long hair, a feature commonly associated with the female, may be the key factor leading to GIC's misclassification.}
  \label{fig:error-2}
\end{figure}

\begin{figure}[H]
  \centering
  \includegraphics[width=.9\textwidth, height=0.4\textheight]{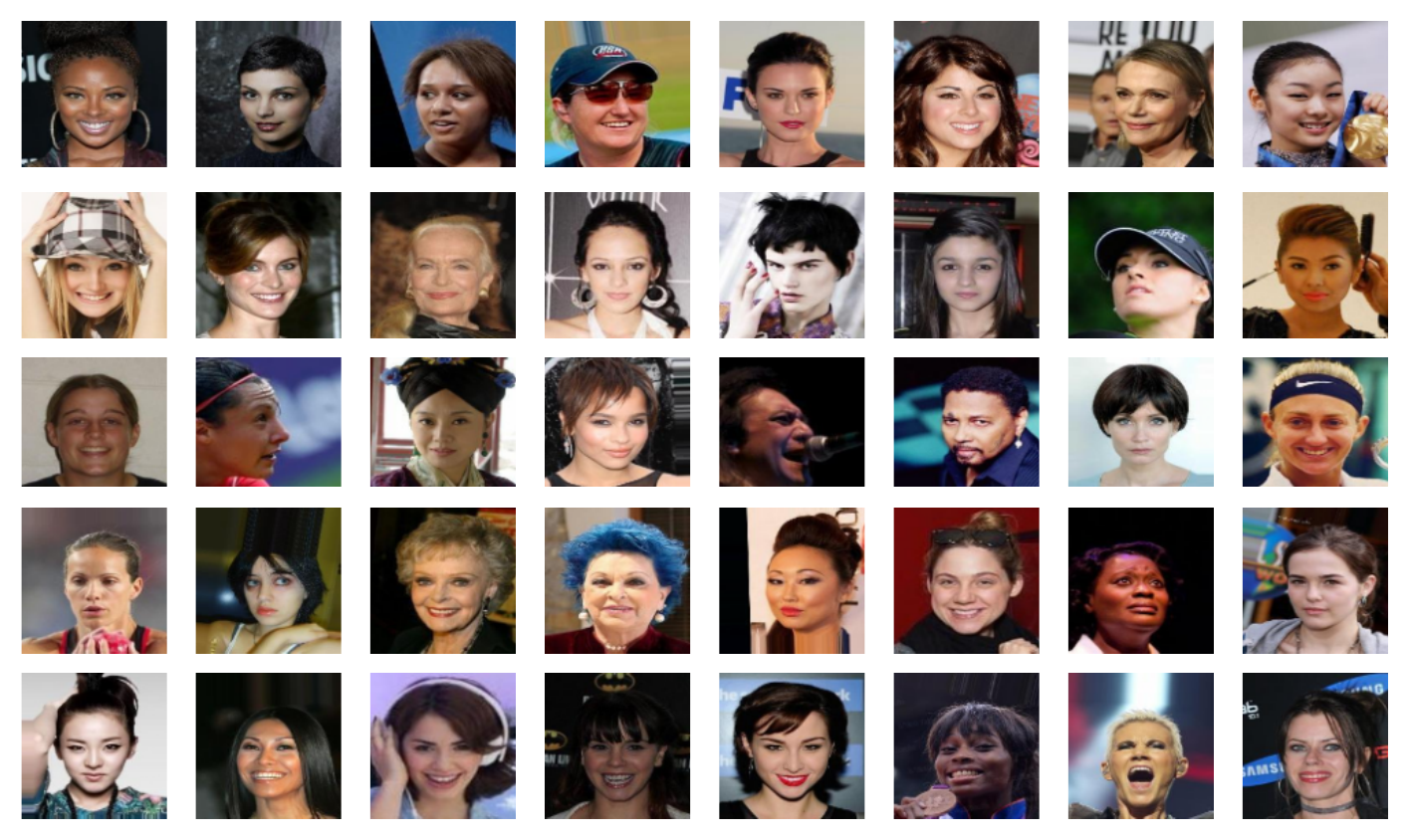}
% width=1.0\textwidth, 
  \caption{More error case on CelebA. Instances misclassified by GIC as male but are actually female. We find that most of these females have short hair or wear hats or headgear that obscures their long hair. The presence of short hair may be the reason for their misclassification by GIC, which aligns with the conclusion shown in \cref{fig:error-2}.}
  \label{fig:error-3}
\end{figure}
\subsection{The Construction of Comparison Data}
\label{app: The construction of comparison data}
In this section, we discuss the methods for constructing comparison data. In \cref{The Construction of Comparison Data}, we proposed three methods for obtaining comparison data. In the experimental section (Section \ref{Experiments}), we used labeled/unlabeled validation sets as labeled/unlabeled comparison data from non-training datasets. The significant improvement in worst-group accuracy indicates the effectiveness of these two methods of obtaining comparison data.

To support the effectiveness of the method of constructing comparison data by non-uniform sampling from the training set, we first conduct additional experiments on CMNIST. Following a similar approach to JTT, we train an ERM model on the training set and divide the samples into an error set (misclassified samples) and a non-error set (correctly classified samples). The error set is typically composed of samples from the minority group with spurious associations \cite{liu2021just}. We then sample an equal number of samples from both the error set and non-error set. For instance, when the sampling ratio is set to 1\%, we sample 300 (1\% of 30,000) samples from each set, resulting in a total of 600 samples as the comparison data. These comparison data samples are combined with the remaining 29,400 training samples for GIC training. After obtaining the spurious attribute classifier $f_{\mathrm{GIC}}$, we directly evaluate its prediction accuracy on the entire training dataset, as shown in \cref{fig:construction-ratio}. We observe that as the sampling rate gradually increases, leading to a growth in the number of comparison data samples, GIC's capability to infer spurious features also enhances. Remarkably, on CMNIST, when the sampling rate is around 10\%, the performance of GIC based on training data for inferring spurious features aligns with that of GIC using validation data directly.
% We observe that the method of constructing comparison data non-uniformly from the training set remains effective. Although the GIC performance is weaker when the sampling rate is too small and the number of comparison data samples is small, when the sampling rate reaches around 10\% (with 6000 comparison data samples), it can match the performance of using validation/test sets, achieving 100\% accuracy in inferring spurious features.
\begin{figure}[H]
  \centering
  \includegraphics[width=.3\textwidth, height=0.15\textheight]{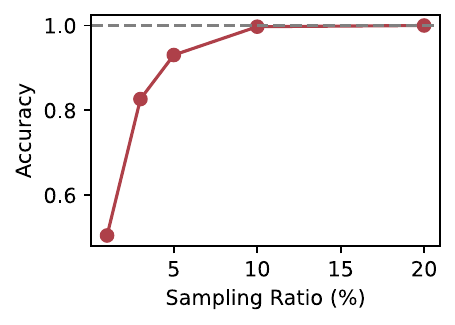}
  \caption{Test accuracy on spurious attributes using non-uniform sampling from the training dataset. The dashed line represents the accuracy of GIC's spurious feature label inference using labeled validation data.}
  \label{fig:construction-ratio}
\end{figure}

% \begin{figure}[H]
%   \centering
%   \includegraphics[width=.3\textwidth, height=0.15\textheight]{png/accuracy_reconstrct.pdf}
%   \includegraphics[width=.3\textwidth, height=0.15\textheight]{png/accuracy_reconstrct_wb.pdf}
%   \includegraphics[width=.3\textwidth, height=0.15\textheight]{png/accuracy_reconstrct_celebA.pdf}
%   \caption{\textcolor{blue}{Test Accuracy on Spurious Attributes Using non-uniform Sampling from the Training Dataset. The dashed line represents the accuracy of GIC's spurious feature label inference using labeled validation data. We observed that as the sampling rate gradually increases, leading to a growth in the number of comparison data samples, GIC's capability to infer spurious features also enhances. Remarkably, on CMNIST, when the sampling rate is around 10\%, the performance of GIC based on training data for inferring spurious features aligns with that of GIC using validation data directly.}}
%   \label{fig:construction-ratio}
% \end{figure}

% \begin{figure}[H]
%   \centering
%   \includegraphics[width=.3\textwidth, height=0.15\textheight]{png/accuracy_reconstrct.pdf}
% % width=1.0\textwidth, 
%   \caption{Test Accuracy on Spurious Attribute of Reconstruction method. we observed that when the sampling rate reaches around 10\% (with 6000 comparison data samples), GIC can fully predict spurious features.}
%   \label{fig:construction}
% \end{figure}

We also conduct additional experiments on CMNIST, Waterbirds, and CelebA to explore the strategy of non-uniform sampling from training data as comparison data. Notably, we observe that for Waterbirds and CelebA, relying solely on training data would significantly limit the sample size of comparison data. Therefore, we combine the training and validation data as a new dataset for constructing comparison data for the Waterbirds and CelebA datasets. Specifically, we mix the training and validation data, then split them into new training and validation sets at a 50\% ratio. We train on the new training set using the ERM method and infer labels on the new validation set, resulting in an error set (misclassified samples) and a non-error set (correctly classified samples). Similarly, we train the model on the validation set and identified the error set and non-error set on the training set through group inference. We then combine the error and non-error sets and sample a fixed proportion (50\%) from each set to serve as comparison data, with the remainder as training data. We then apply GIC to infer group labels and learn invariant features. To further enhance performance, we implement a boosting approach by repeating the aforementioned process twice. In the second iteration, we base it on the GIC model obtained from the first round of training to infer group labels and construct the error set and non-error set.

We further report the worst-group accuracy of GIC when the comparison data is generated by non-uniform sampling from the training dataset, using Mixup as the downstream algorithm. We refer to this variant of GIC as $\mathrm{GIC}_{\mathcal{C}_{t}}$-M. The experimental results in \cref{tab:acc-construct} demonstrates that constructing comparison data by non-uniform sampling from the training dataset is an effective strategy, as the worst-group accuracy is close to that achieved by the GIC approach using labeled/unlabeled validation data as the comparison dataset, particularly on CMNIST, where the performances are almost comparable.
\begin{table}[H]
% \vspace{-0.2cm}
\setlength{\abovecaptionskip}{-0.01cm}
\centering
\caption{{Comparison of GIC's worst-group accuracy using different sources of comparison data.}}
\vskip 0.1in
% \setlength{\tabcolsep}{1.5pt} %
% {
{
\begin{tabular}{@{}lccccccccccc@{}}
\toprule
{Method} &  {CMNIST} &  {\makecell{Waterbirds}}& {\makecell{CelebA}}\\
% &Worst & Worst& Worst\\
\midrule
\rowcolor{gray!25}\multicolumn{4}{c}{\textit{validation}} \\
\midrule
$\mathrm{GIC}_{\mathcal{C}_y}$-M &72.2$\pm \scriptstyle{0.5}$ &86.3$\pm \scriptstyle{0.1}$ &89.4$\pm \scriptstyle{0.2}$\\
$\mathrm{GIC}_{\mathcal{C}}$-M  & 71.7$\pm \scriptstyle{0.3}$ & 85.4$\pm \scriptstyle{0.1}$&89.5$\pm \scriptstyle{0.0}$\\
\midrule
\rowcolor{gray!25}\multicolumn{4}{c}{\textit{non-uniform sampling.}} \\
\midrule
$\mathrm{GIC}_{\mathcal{C}_{t}}$-M &72.2 $\pm \scriptstyle{0.1}$ &85.7$\pm \scriptstyle{0.3}$&87.5$\pm \scriptstyle{0.7}$\\
% 85.9/85.4&88.2/86.8
\bottomrule
\end{tabular}}
\label{tab:acc-construct}
 % \vspace{-10pt}
\end{table}
\subsection{The Necessity of Considering Unlabeled Datasets as Comparison Data.}
\label{app:The Necessity of Considering Unlabeled Datasets as Comparison Data}
Additionally, we want to emphasize the necessity of considering unlabeled comparison data, such as sampling from the test set. In the real world, unlabeled data is often more cost-effective and easier to obtain than labeled data. By utilizing a large amount of unlabeled data as comparison data, the applicability of GIC can be greatly expanded. Our results in \cref{fig:construction} demonstrate that when GIC trains a spurious feature classifier using labeled comparison data (with sample sizes ranging from 10 to 1000), its accuracy in predicting spurious features and improving the accuracy of the worst-performing group is significantly worse than when using unlabeled but larger comparison data (with sample sizes ranging from 3000 to 10000). This experiment highlights the importance of considering unlabeled comparison data.
% \textcolor{blue}{However, we also observe that the performance of constructing comparison data from the training data tends to be inferior to those that use the validation set as comparison data. According to \cref{fig:construction}, we suspect that the underlying reason may be the limited sample size of comparison data when constructed from the training data, which could result in diminished GIC performance}
\begin{figure}[H]
  \centering
  \includegraphics[width=.3\textwidth, height=0.15\textheight]{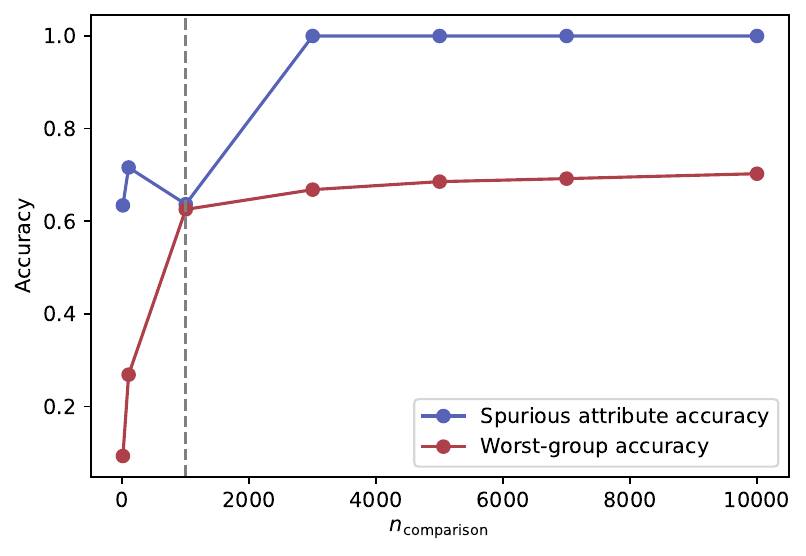}
% width=1.0\textwidth, 
  \caption{The worst-group accuracy and spurious attributes accuracy on CMNIST. On the left side of the dashed line, labeled comparison data (ranging from 10 to 1000) is utilized, but due to the small sample size, GIC demonstrates poor performance in inferring groups and improving the worst-group accuracy. On the right side of the dashed line, a larger unlabeled validation set is used as comparison data, resulting in a noticeable improvement in GIC's performance.}
  \label{fig:construction}
\end{figure}

\subsection{Group Distribution Discrepancy Analysis}
\label{app:distribution diff}

We then discuss the impact of differences in group distributions between the comparison data and training data on the effectiveness of GIC. As mentioned in \cref{Discussion}, even slight differences in group distributions can force the spurious term to learn invariant features. Although GIC can still learn the spurious attributes and infer group labels in the presence of slight differences, we want to understand the relationship between the degree of differences and GIC's performance. To answer this question, we use the CelebA dataset as an example. As described in \cref{app:Dataset Details}, inferring spurious attributes on the CelebA dataset can be challenging because the training data and comparison data (validation set) have very similar group distributions. We keep the training set and test set fixed and adjust the group distribution of the comparison data (validation set) using the oracle group labels. We try four different sets of group distributions for the comparison data, and their detailed group distributions are displayed in \cref{fig:diff-tf}. We then train GIC using the training data and these four different sets of comparison data while keeping the parameters and model consistent, and calculate the worst-group accuracy on the test set.

\begin{table}[H]
\centering
    \caption{The result after adjusting the group distribution using the group labels predicted by GIC.}
    \vskip 0.1in
    \label{tab:diff-adjust}
    \begin{tabular}{@{}lccccccc@{}}
    \toprule
    \multirow{2}{*}{Method} & \multirow{2}{*}{\makecell{Readjust}} & \multicolumn{2}{c}{\makecell{CelebA}} \\
    &  & Avg. & Worst\\
    \midrule
    \rowcolor{gray!25}\multicolumn{4}{c}{\textit{subsample}} \\
    \midrule
    $\mathrm{GIC}_{\mathcal{C}_y}$ & $\times$  & 91.1$\pm \scriptstyle{0.1}$ & 86.1$\pm \scriptstyle{2.2}$\\
    \midrule
    $\mathrm{GIC}_{\mathcal{C}_y}$ & $\checkmark$  &92.1$\pm \scriptstyle{0.3}$&  89.1$\pm \scriptstyle{0.9}$\\
    \bottomrule
    \end{tabular}
\end{table}

\begin{figure}[H]
\centering
\includegraphics[width=0.5\textwidth, height=0.13\textheight]{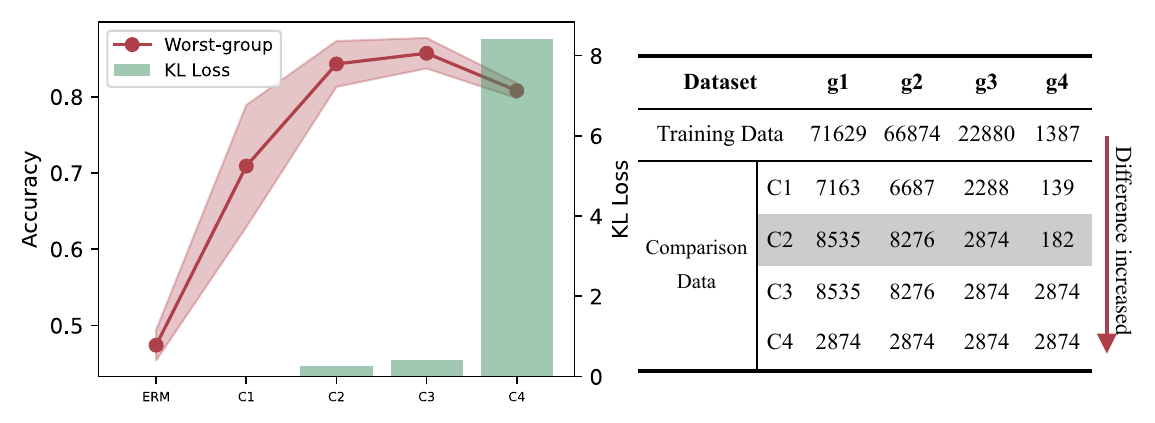}
\caption{Increasing differences in group distribution encourage GIC to better infer group labels, further improving the performance of the worst group.}
\label{fig:diff-tf}
\end{figure}

% \begin{figure}[H]
%   \centering
%   \includegraphics[width=.5\textwidth, height=0.13\textheight]{png/diff_tf.pdf}
% % width=1.0\textwidth, 
%   \caption{Distribution Difference on CelebA}
%   \label{fig:diff-tf}
% \end{figure}
In \cref{fig:diff-tf}, we observe that as the differences in group distributions between the comparison data and training data increase, the worst-group performance improves. Additionally, we calculate the KL Loss for each comparison data and find that as the differences in group distributions increase, the final KL Loss also increases. This suggests that GIC has the ability to accurately capture the differences in group distributions between the comparison data and training data. 

It is worth noting that the group distribution of comparison data C1 is almost identical to the training data (KL loss close to 0), yet its worst-group accuracy still outperforms the worst-group accuracy of the ERM model. We believe that this is because when there is no group distribution difference, whether the spurious term forces GIC to learn invariant features or spurious features is a completely random event with equal probability. However, due to the presence of spurious correlations, GIC is more inclined to prioritize learning spurious features, thereby improving the worst-group accuracy. However, such events of prioritizing learning spurious features are not stable, which is why we observe larger variances in the results corresponding to the C1 data compared to other comparison datasets.

The phenomenon observed in \cref{fig:diff-tf} (that larger differences can improve GIC performance) inspires us to attempt adjusting the group distribution of the comparison data again using the group label results predicted by GIC. Specifically, we use $f_{\mathrm{GIC}}$ trained on the celebA dataset to adjust the comparison data (validation set) again. This involves Upsample the two smallest groups to the same sample size as the second largest group, and then retraining GIC using the new comparison data and training data. Table \ref{tab:diff-adjust} reports the GIC obtained after re-adjusting the comparison data, and we observe that the accuracy of the worst-group is further improved through this operation.

\subsection{Compute Resources and Training Time of GIC}
\label{GIC compute resources and training time}

All experiments for CMNIST, Waterbirds, CelebA and CivilComments were run on a single NVIDIA GeForce RTX 4090 GPU.

Regarding the runtime of GIC, additional computational overhead compared to methods inferring spurious labels via ERM stems from inputting data representations $\mathbf{z}$ into the GIC model $f_{\mathrm{GIC}}$ for both training and predicting the spurious feature labels $\hat y^{c}_{s,\mathbf{w}}$. This supplementary step should not necessitate substantial computation cost. As the dimensionality of data representation $\mathbf{z}$ is reduced relative to the raw data, and the experiments have evidenced that a simple single-layer neural network for $f_{\mathrm{GIC}}$ suffices in delivering promising performance. Concurrently, it's noted in \cref{app:Training Details} that the training epoch $K$  can be contained within 20 epochs for all four experimental datasets. We further calculate the total time required for training the spurious feature classifier $f_{\mathrm{GIC}}$ and predicting the spurious feature labels $\hat y^{c}_{s,\mathbf{w}}$ across the four experimental datasets, manifesting that the computational cost introduced by GIC is rather minimal.

\begin{table}[H]
% \vspace{-0.2cm}
% \setlength{\abovecaptionskip}{-0.01cm}
\centering
\caption{The average runtime for training $f_{\mathrm{GIC}}$ and predicting $\hat y^{c}_{s,\mathbf{w}}$. The timing commences subsequent to obtaining the data representation $\mathbf{z}$ and concludes upon acquiring the predicted $\hat y^{c}_{s,\mathbf{w}}$.}
\vskip 0.1in
{
\begin{tabular}{@{}lccc|cc|cc|cccc@{}}
\toprule
{Dataset} & Average Time\\
\midrule
CMNIST&30.1s \\
Waterbirds&60.0s\\
CelebA & 519.9s\\
CivilCommnets &901.0s \\

\bottomrule
\end{tabular}}
\label{tab:time}
 % \vspace{-10pt}
\end{table}

\end{document}